\documentclass{article}
\usepackage{PRIMEarxiv}
\usepackage{hyperref}
\usepackage{adjustbox}
\usepackage{lipsum} 
\usepackage{cite}
\usepackage{import}
\usepackage{nicematrix}
\usepackage{graphicx}%
\usepackage{multirow}%
\usepackage{amsmath,amssymb,amsfonts}%
\usepackage{amsthm}%
\usepackage{mathrsfs}%
\usepackage{xcolor}%
\usepackage{textcomp}%
\usepackage{manyfoot}%
\usepackage{booktabs}%
\usepackage{algorithm}%
\usepackage{algorithmicx}%
\usepackage{algpseudocode}%
\usepackage{listings}%
\usepackage{adjustbox}
\usepackage{caption, subcaption}
\usepackage{xcolor}
\usepackage[utf8]{inputenc} 
\usepackage[T1]{fontenc}    
\usepackage{hyperref}       
\usepackage{url}            
\usepackage{booktabs}       
\usepackage{amsfonts}       
\usepackage{nicefrac}       
\usepackage{microtype}      
\usepackage{lipsum}
\usepackage{fancyhdr}       
\usepackage{graphicx}  

\newtheorem{definition}{Definition}[section]

\newtheorem{theorem}{Theorem}[section]
\newtheorem{remark}{Remark}[section]

\title{Randomized Approach to Matrix Completion: \\ Applications in Recommendation Systems \\ and Image Inpainting\thanks{This manuscript corresponds to the accepted version of the article published in Machine Learning.
The final authenticated Version of Record is available online at: \url{https://doi.org/10.1007/s10994-026-06995-2}}}

\author{
  Antonina Krajewska\\
  NASK National Research Institute\\
  Warsaw, Poland\\
  \texttt{antonina.krajewska@nask.pl}\\
  \And
  Ewa Niewiadomska-Szynkiewicz\\
  Warsaw University of Technology\\
  Warsaw, Poland\\
  \texttt{ewa.szynkiewicz@pw.edu.pl}\\
}

\begin{document}

\maketitle

\begin{abstract}We present a novel method for matrix completion, specifically designed for matrices where one dimension is significantly larger than the other. Our Columns Selected Matrix Completion (CSMC) method combines Column Subset Selection with Low-Rank Matrix Completion to efficiently reconstruct incomplete datasets. CSMC substantially reduces computational cost while preserving the solution quality of state-of-the-art convex matrix completion techniques. Each stage of CSMC involves solving a convex optimization problem. We introduce two algorithmic implementations of CSMC, each tailored to problems of different scales. A formal analysis is provided, outlining the necessary assumptions and the probability of exact recovery. To evaluate the impact of matrix size, rank, and the ratio of missing entries on solution quality and runtime, we conducted experiments on synthetic data. The method was also applied to two real-world tasks: recommendation systems and image inpainting. Our results show that CSMC achieves substantial runtime improvements while maintaining competitive accuracy compared to leading convex optimization-based methods.
\end{abstract}




\section{Introduction}

Images, time series, and graphs are examples of data types that can naturally be represented as real-valued matrices. In practice, these matrices often contain missing entries due to data collection costs, sensor errors, or privacy constraints. To fully leverage matrix-based datasets, it is essential to address the missing data problem. Matrix completion (MC) methods often exploit the low-rank structure of data matrices, which captures inherent correlations and enables effective recovery of missing entries.

Matrix completion has found applications in diverse fields such as signal processing, sensor localization, medicine and genomics \cite{Cai23, Cai22cur, le2018light, Mai25}. Its significance was underscored by the Netflix Prize challenge, which focused on predicting unknown ratings using collaborative filtering based on the preferences of many users \cite{Moitra18, Zhang2024, Li23}. The task involves completing the sparse user-movie rating matrix $\mathbf{M} \in \mathbb{R}^{n_1 \times n_2}$, where $n_1$ (the number of users) is significantly larger than $n_2$ (the number of movies). This creates an imbalanced matrix in which most entries are missing.

Singular Value Decomposition (SVD) is fundamental to many matrix completion algorithms. Any matrix $\mathbf{M} \in \mathbb{R}^{n_1 \times n_2}$ can be written as  
\begin{align} 
\mathbf{M} = \mathbf{U}\mathbf{\Sigma}\mathbf{V}^T,
\end{align}
where $\mathbf{U} \in \mathbb{R}^{n_1 \times n_1}$ and $\mathbf{V} \in \mathbb{R}^{n_2 \times n_2}$ are orthogonal, and $\mathbf{\Sigma} \in \mathbb{R}^{n_1 \times n_2}$ is diagonal with nonnegative singular values in decreasing order. The rank of $\mathbf{M}$ equals the number of nonzero singular values.

Matrix completion exploits this low-rank structure by searching for the lowest-rank matrix consistent with the observed entries. Since rank minimization is NP-hard~\cite{Moitra18, fazel10}, a common relaxation is to minimize the nuclear norm—the sum of the singular values—which yields a convex program with strong recovery guarantees~\cite{Moitra18}. Let $\Omega \subseteq \{ 1, \cdots , n_1 \} \times \{ 1, \cdots n_2 \}$ be the set of indices of the known entries, the convex optimization problem for the exact MC has the following form
\begin{align}\label{eq:exact_nuclear_norm}
&\left| \begin{aligned} 
&\mathrm{minimize}&& \|\mathbf{X}\|_{*}
\\
&\mathrm{s.t.} && 
\mathcal{R}_{\Omega}(\mathbf{X}-\mathbf{M}) &= 0, \\
\end{aligned} \right. 
\end{align}
\noindent
where $\|\mathbf{X}\|_{*}$ denotes nuclear norm of the decision variable $\mathbf{X}$, $\mathcal{R}_{\Omega} : \mathbb{R}^{n_1 \times n_2} \rightarrow  \mathbb{R}^{n_1 \times n_2}$ sets all matrix elements to $0$, except those in the $\Omega$ set, e.g.

\begin{align}\label{eq:sampling_operator}
\mathcal{R}_{\Omega}(\mathbf{X}) = \sum_{(i, j) \in \Omega} x_{ij} \mathbf{e}_i \mathbf{e}_j^T, 
\end{align}
\noindent
where $\mathbf{e}_i$ denotes standard basis vector with $1$ on the $i$-th coordinate and $0$ elsewhere. Nuclear norm minimization is a widely adopted technique for low-rank matrix completion, known for its theoretical guarantees and strong empirical performance, which contribute to the method’s transparency and interpretability \cite{Recht09, Chen18, Moitra18}. Under suitable conditions, it can accurately recover the underlying matrix from a relatively small number of observed entries, on the order of 
$\mathcal{O}(nr\ln^2{n})$, where $n=n_1+n_2$. The optimization problem can be solved in polynomial time in the matrix dimension using off-the-shelf semidefinite programming (SDP) solvers or tailored first-order methods \cite{fazel10, Chen18, Moitra18}. However, for large-scale problems, solving these SDP programs remains computationally challenging due to high memory requirements \cite{Chen18}.

This paper introduces a novel method for the low-rank matrix completion problem, called Columns Selected Matrix Completion (CSMC). CSMC addresses the efficiency limitations of nuclear norm minimization while maintaining its theoretical guarantees. Our approach integrates two areas of modern linear algebra: it first employs the Column Subset Selection (CSS) algorithm to reduce the size of the MC problem. In the second stage, the recovered and previously known entries are used to reconstruct $\mathbf{M}$ by solving a standard least squares problem \cite{Woodruff14}. CSMC is particularly effective when one dimension of $\mathbf{M}$ is much larger than the other, such as in the Netflix Prize scenario. To the best of our knowledge, this is the first application of CSS to improve the computational efficiency of matrix completion under the uniform sampling assumption. Prior approaches either require access to fully observed columns~\cite{Krishnamurthy14, Xu15} or modify the sampling scheme in ways that introduce additional computational overhead~\cite{Ward18}.

Theoretical analysis shows that CSMC guarantees exact recovery under standard incoherence conditions with high probability, requiring the same number of observed entries, 
$\mathcal{O}(nr\ln^2{n})$. Our experiments on both synthetic and real-world datasets demonstrate competitive performance compared to state-of-the-art baselines. Preliminary results were presented in a shorter conference version \cite{Krajewska24}. In contrast, this paper provides the first comprehensive theoretical analysis of the CSMC method, along with an in-depth empirical evaluation on real and synthetic data.

\subsection{Notation and abbreviations}
For a matrix $\mathbf{X} \in \mathbb{R}^{n_1 \times n_2}$ we denote by $\mathbf{x}_i$ its $i$-th column, and by $x_{ij}$ its $(i, j)$-th entry. Given the set of column indices $\mathcal{I} \in [n_2]$ we use  $\mathbf{X}_{:, \mathcal{I}}$ to denote the column submatrix.  We consider three types of matrix norms. The spectral norm of $\mathbf{X}$ is denoted by  $\| \mathbf{X} \|_2$. The Frobenius and the nuclear norm are denoted by $\| \mathbf{X} \|_F$ and $\| \mathbf{X} \|_{*}$ respectively. Table \ref{tab:notation} contains essential symbols and acronyms.
\begin{table}[h]
\caption{Notation and abbreviations}\label{tab:notation}%
\centering
\begin{tabular}{@{}cl@{}}
\toprule
\textbf{Notation} & \textbf{Description} \\
\midrule
$\Omega$ & set of indices of the observed entries \\
$\mathcal{R}_{\Omega}$ & sampling operator (\ref{eq:sampling_operator}) \\
$\mathcal{I}$ & column indices for the column submatrix $\mathbf{C}$ \\
$\mathbf{X}_{:, \mathcal{I}}$ & column submatrix of $\mathbf{X}$ indicated by $\mathcal{I}$ \\
$\mathbf{X}^{\dagger}$ & Moore-Penrose inverse of $\mathbf{X}$ \\
$\rho$ & percentage of sampled entries \\
$\alpha$ & percentage of sampled columns in the first stage \\
$r$ & rank of $\mathbf{M}$ \\
$\mu_0(\mathbf{M})$ & incoherence parameter of $\mathbf{M}$ \\
CSS & Column Subset Selection \\
CSMC & Columns Selected Matrix Completion \\
CSNN & Columns Selected Nuclear Norm Minimization \\
CSPGD & Columns Selected Proximal Gradient Descent \\
NN & Nuclear Norm Minimization \\
PGD & Proximal Gradient Descent \\
MF & Matrix Factorization \\
\bottomrule
\end{tabular}
\end{table}

\subsection{Outline of the paper}

The rest of the paper is organized as follows. Section~2 reviews related work and highlights our contributions. Section~3 introduces the proposed method and provides the formal analysis, including Theorem~\ref{thm:csmc_g}, which establishes a lower bound on the probability of perfect matrix reconstruction, together with conditions on the minimum number of randomly selected columns and the minimum number of observed entries required to achieve this bound. Section~4 presents experimental results on both synthetic and real datasets, confirming the effectiveness of the proposed approach. Section~5 contains proofs and supporting theorems, and Section~6 concludes the paper and outlines directions for future research.

\section{Related work and contribution}

In this section, we discuss existing approaches to matrix completion, focusing on their theoretical guarantees and efficiency in terms of the number of observed entries required for accurate recovery, commonly referred to as sample complexity. We also review the Column Subset Selection (CSS) problem and its applications to MC. Finally, we present our original contributions that integrate these two lines of research to enhance both computational efficiency and scalability.

\subsection{Low-rank matrix completion}

The low-rank assumption is grounded in the observation that the real-world datasets inherently possess, or can be effectively approximated by, a low-rank matrix structure \cite{Chen18}. It plays a pivotal role in enhancing the efficiency of algorithms. The algorithms for the low-rank matrix completion typically leverage convex or non-convex optimization in Euclidean spaces \cite{jain13, Recht09, Mazumder2010a, Candes08, Hardt2014}, or employ Riemannian optimization \cite{cambier2016robust, bian2024preconditioned}. The majority of research has focused on uniform or Bernoulli sampling models, which usually require the matrix to be incoherent (Definition~\ref{def:coherence}). Among these approaches, convex methods are the most developed and provide strong, well-understood theoretical guarantees, with a sample complexity of $\mathcal{O}(nr\ln^2 n)$ for incoherent matrices, where $n=n_1+n_2$. The nuclear norm relaxation draws inspiration from the success of compressed sensing \cite{Chen18, Moitra18, fazel10} in which $\ell_1$ norm minimization is used to recover sparse signal instead of  $\ell_0$ norm. The matrix rank may be expressed as the $\ell_0$ norm of the singular values vector $\mathbf{\sigma}$, while the nuclear norm is defined as its  $\ell_1$ norm.

To address noisy data and prevent model overfitting, Mazumder et al. \cite{Mazumder2010a} consider a regularized optimization problem,

\begin{align}\label{eq:lagrange}
  \underset{\mathbf{X} \in \mathbb{R}^{n_1 \times n_2}}{\mathrm{minimize\ }} \frac{1}{2}  \| \mathcal{R}_{\Omega}(\mathbf{X}) - \mathcal{R}_{\Omega}(\mathbf{M})\|_{F} ^2 + \lambda \|\mathbf{X}\|_{*},
\end{align}
\noindent
where $\lambda$ is a regularization paramater. The problem (\ref{eq:lagrange}) benefits from the closed form of the proximal operator for the objective function given by Singular Value Thresholding (SVT) and is solved with Proximal Gradient descent (PGD) \cite{Mazumder2010a}. However, PGD requires calculation of SVD at each iteration, making it computationally expensive.  

Another class of convex optimization techniques focuses on minimizing the Frobenius norm error for the observed entries while applying constraints on the nuclear norm. It can be solved using the Frank-Wolfe algorithm \cite{Jaggi13, Jing23}. Despite its low computational cost per iteration, the Frank-Wolfe algorithm exhibits slow convergence rates. Consequently, numerous efforts have been undertaken to develop optimized variants of the algorithm \cite{Yurtsever17}.

Non-convex optimization methods are often used in practice because they are faster than nuclear norm minimization. These methods aim to minimize the least squares error on the observed entries while constraining the rank of the solution using Matrix Factorization (MF) \cite{Chen18, tong2021accelerating, jain14, li2024novel},

\begin{align}\label{eq:burer_monteiro}
\underset{\mathbf{L} \in \mathbb{R}^{n_1 \times k}, \mathbf{R} \in \mathbb{R}^{n_2 \times k}}{\mathrm{minimize}} \frac{1}{2}\| \mathcal{R}_{\Omega}(\mathbf{M}) - \mathcal{R}_{\Omega}(\mathbf{L}\mathbf{R}^T)\|^2_{F}.
\end{align}

This formulation requires specifying the parameter $k$, which must be tuned since the true rank of $\mathbf{M}$ is unknown. Chen and Chi \cite{Chen18} categorize three classes of methods for solving (\ref{eq:burer_monteiro}): Alternating Minimization, Gradient Descent, and Singular Value Projection (SVP) \cite{jain10, jain14}. For Alternating Minimization, Jain et al. \cite{jain13} and Hardt \cite{Hardt2014} provided provable guarantees. However, these guarantees are weaker in terms of the matrix incoherence, rank, and condition number of $\mathbf{M}$. Among gradient descent methods, Scaled Gradient Descent (ScaledGD) achieves the best sampling complexity, $\mathcal{O}(r^2 n \ln n)$ for incoherent matrices, which is worse than nuclear norm minimization by a factor of $r$                     \cite{tong2021accelerating}. The sampling complexity of SVP is even higher, $\mathcal{O}(r^5 n \ln^3 n)$ for incoherent matrices \cite{jain10, jain14}.

\subsection{Column Subset Selection}

Many commonly employed low-rank models, such as Principal Component Analysis (PCA), Singular Value Decomposition (SVD), and Nonnegative Matrix Factorization (NMF) lack clear interpretability, leading to research on alternatives like Column Subset Selection (CSS) and CUR decomposition \cite{ sorensen2016deim, voronin2017efficient}.
CSS provides low-rank representation using the column submatrix of the data matrix $\mathbf{M} \in \mathbb{R}^{n_1 \times n_2}$. Formally, given a matrix $\mathbf{M} \in \mathbb{R}^{n_1 \times n_2}$ and a positive integer $d$, CSS seeks for the column submatrix $\mathbf{C} \in \mathbb{R}^{n_1 \times d}$ such that

\begin{align}
\| \mathbf{M} - P_{C}(\mathbf{M}) \|_{\xi}
\end{align}
\noindent
is minimized over all possible $\binom{n_2}{d}$ choices for the matrix $\mathbf{C}$. Parameter $\xi \in \{2, F\}$, where  $\xi= 2$ denotes the spectral and $\xi= F$ the Frobenius norm. Here,
\begin{align}
P_{C} = \mathbf{C}\mathbf{C}^{\dagger},
\end{align}
\noindent
is the projection onto the $d$-dimensional space spanned by the columns $\mathbf{C}$ and $\mathbf{C}^{\dagger}$ denotes the Moore-Penrose pseudoinverse. On the other hand, CUR decomposition explicitly selects a subset of columns and rows from the original matrix to construct the factorization \cite{sorensen2016deim}. 

Extensive research focuses on fully-observed matrices, which may not be practical for real-world applications \cite{sorensen2016deim, voronin2017efficient}. Several works provide comprehensive overviews of CUR decomposition methods. For instance, \cite{hamm2020perspectives} offers a detailed characterization of various forms of CUR, while \cite{hamm2020stability} reviews different sampling strategies employed to construct CUR decompositions. Bridging the research gap between CSS, CUR and MC is a recently recognized challenge  \cite{Wang17, Ward18, Krishnamurthy14}. In this context, two important questions emerge. The first one is about selecting meaningful columns (or rows) from partially observed data and applying CSS and CUR in MC algorithms. Our method benefits from the fact that under standard assumptions about matrix incoherence, a high-quality column subset can be obtained by sampling each column according to the uniform distribution \cite{Xu15}. On the other hand, in the case of the coherent matrices, Wang and Singh \cite{Wang17} propose active sampling to dynamically adjust column selection during the sampling process, albeit with increased computational complexity.

The application of CUR and CSS in MC tasks typically involves altering the sampling scheme to relax assumptions about the matrix or reduce sample complexity \cite{Krishnamurthy14, Ward18, Cai22cur, Xu15}. To the best of our knowledge, our approach is the first to apply CSS for reducing memory usage and runtime in nuclear norm minimization–based MC.

Eftekhari et al. \cite{Ward18} introduce the MC2 model, which begins by sampling part of the matrix entries uniformly at random. Using these samples, the row and column leverage scores are estimated, and an additional set of entries is then sampled according to the estimated scores. The full matrix is subsequently recovered through nuclear norm minimization. For incoherent matrices, MC2 achieves the same sampling complexity as nuclear norm minimization when the entries are uniformly sampled. For certain classes of coherent matrices, it reduces the number of observed entries needed for exact recovery. Despite this advantage, MC2 has the same time and space complexity as nuclear norm minimization and incurs the extra cost of estimating leverage scores.

Krishnamurthy and Singh \cite{Krishnamurthy14} propose a fast MC algorithm which improves the sample complexity of nuclear norm minimization by avoiding the assumption on the row-space coherence. It leverages adaptive column sampling: based on partial observations and the estimated column space, the algorithm determines whether a column should be fully observed or approximated. Nevertheless, the approach still requires some columns to be fully observed, which can be impractical in real-world settings—for instance, in recommendation systems, where users are unlikely to rate all items.

Cai et al. \cite{Cai22cur} showcase the potential of CUR in low-rank matrix completion, introducing a novel sampling strategy called Cross-Concentrating Sampling (CCS) and a fast non-convex Iterative CUR Completion (ICURC) algorithm. Still, CCS requires $\mathcal{O}(r^2n\ln^2(n))$ observed entries, where $n =\max\{ n_1, n_2 \}$, which is a factor of $r$ worse than the nuclear norm minimization-based approach. The optimization problem solved by ICURC imposes constraints on the rank of the matrix, thus introducing an additional parameter that requires tuning.

Our approach is closely related to the work of Xu et al., who proposed CUR+, an algorithm for matrix completion from partial observations \cite{Xu15}. CUR+ constructs a low-rank approximation using (i) randomly sampled rows and columns and (ii) a subset of observed entries, and it guarantees recovery of incoherent matrices from $\mathcal{O}(rn \ln r)$ samples. However, CUR+ assumes that the sampled row and column submatrices are fully observed, which is impractical in many standard real-world applications. To enable a fair comparison, we adapted it into a two-stage method (CUR+Nuc): first, completing the sampled submatrices with nuclear norm minimization, and then applying CUR+ regression. In Experiment S~III, we implemented this adaptation and compared it with our method. In contrast to CUR+ and CUR+Nuc, CSMC requires filling only the column submatrix, avoids computing singular vectors, and does not require specifying the rank. These differences remove both the sensitivity of CUR+ to the choice of $k$ and the cost of eigenvector computations.

\subsection{Our contributions}

This paper addresses the challenge of completing rectangular matrices by integrating principles from Column Subset Selection and Low-rank Matrix Completion. The proposed approach solves convex optimization problems while enhancing computational efficiency and scalability. Our main contributions are as follows:

\begin{enumerate}
    \item We introduce a two-stage matrix completion method, CSMC, designed for recovering rectangular matrices with highly unbalanced dimensions ($n_1 \ll n_2$). This structure is particularly suited for large-scale, asymmetric data scenarios common in multi-source data integration. To the best of our knowledge, this is the first application of Column Subset Selection to improve the runtime efficiency of convex matrix completion methods. 
    \item We provide a formal analysis, including Theorem~\ref{thm:csmc_g}, which establishes a lower bound on the probability of exact matrix reconstruction in the second stage of CSMC. The theorem specifies necessary conditions on the number of randomly selected columns ($d = \mathcal{O}(r \ln r)$) and on the sample size of observed entries ($|\Omega| = \mathcal{O}(n_2 r \ln (n_2 r))$) required for accurate recovery of an incoherent rank-$r$ matrix. Moreover, we show that under uniform sampling, the sample complexity of CSMC matches that of nuclear norm minimization, i.e., $\mathcal{O}(r(n_1+n_2)\ln^2(n_1+n_2))$ for incoherent matrix.
    \item We compare CSMC with prior methods that use CSS, both theoretically and experimentally. These differences make CSMC more practical for sparse real-world settings (e.g., recommendation systems) and more efficient in terms of runtime and memory usage.

    \item We develop two algorithms dedicated to problems of different sizes:  Columns Selected Nuclear Norm (CSNN) and Columns Selected Proximal Gradient Descent (CSPGD).
    \item We validate CSMC through numerical experiments on both synthetic and real-world datasets, benchmarking its performance against established matrix completion methods in recommendation systems and image inpainting. Our results demonstrate that CSMC achieves reconstruction quality on par with convex nuclear norm minimization methods, while attaining runtimes comparable to state-of-the-art non-convex approaches.
\end{enumerate}

\section{Columns Selected Matrix Completion (CSMC)}

\subsection{Method overview}
In the CSMC, the completion of the matrix  $\mathbf{M} \in \mathbb{R}^{n_1 \times n_2}$ is divided into two stages. In the first stage, CSMC selects the subset of $d$ columns and fills it with an established MC algorithm. In the second stage, the least squares problem is solved. Each of the presented variants of the CSMC methods adopts the following procedure (Fig. \ref{fig:CSMC}).

\begin{description}
  \item[Stage I: Sample and fill] \hfill \\ CSMC selects $d$ columns of $\mathbf{M} \in \mathbb{R}^{n_1 \times n_2}$ according to the uniform distribution over the set of all indices. Let  $I \subseteq \{1, \cdots, n_2\}$ denote the set of the indices of the selected columns. The submatrix formed by selected columns, $\mathbf{C}:= \mathbf{M}_{:I}$, $\mathbf{C} \in \mathbb{R}^{n_1 \times d}$ is recovered using the chosen matrix completion algorithm.
  \item[Stage II: Solve least squares] \hfill \\ Let  $\mathbf{\hat{C}} \in \mathbb{R}^{n_1 \times d}$ be the output of the Stage I. The CSMC solves the following convex optimization problem,

    \begin{align}\label{eq:randomized_nuclear_norm2}
  \underset{\mathbf{Z} \in \mathbb{R}^{d \times n_2}}{\mathrm{minimize}} \frac{1}{2} \|\mathcal{R}_{\Omega}(\mathbf{M}) - \mathcal{R}_{\Omega}(\mathbf{\hat{C}} \mathbf{Z}) \|_{F}^2,
\end{align}
\noindent
  where $\mathcal{R}_{\Omega} :\mathbb{R}^{n_1 \times n_2} \rightarrow \mathbb{R}^{n_1 \times n_2}$ is the sampling operator (\ref{eq:sampling_operator}). Let $\mathbf{\hat{Z}} \in \mathbb{R}^{d \times n_2}$ be a solution to the problem (\ref{eq:randomized_nuclear_norm2}). The CSMC output with the matrix $\mathbf{\hat{M}} = \mathbf{\hat{C}} \mathbf{\hat{Z}}$.
\end{description}

\begin{figure}[H]
     \centering
    \includegraphics[width=0.8\textwidth]{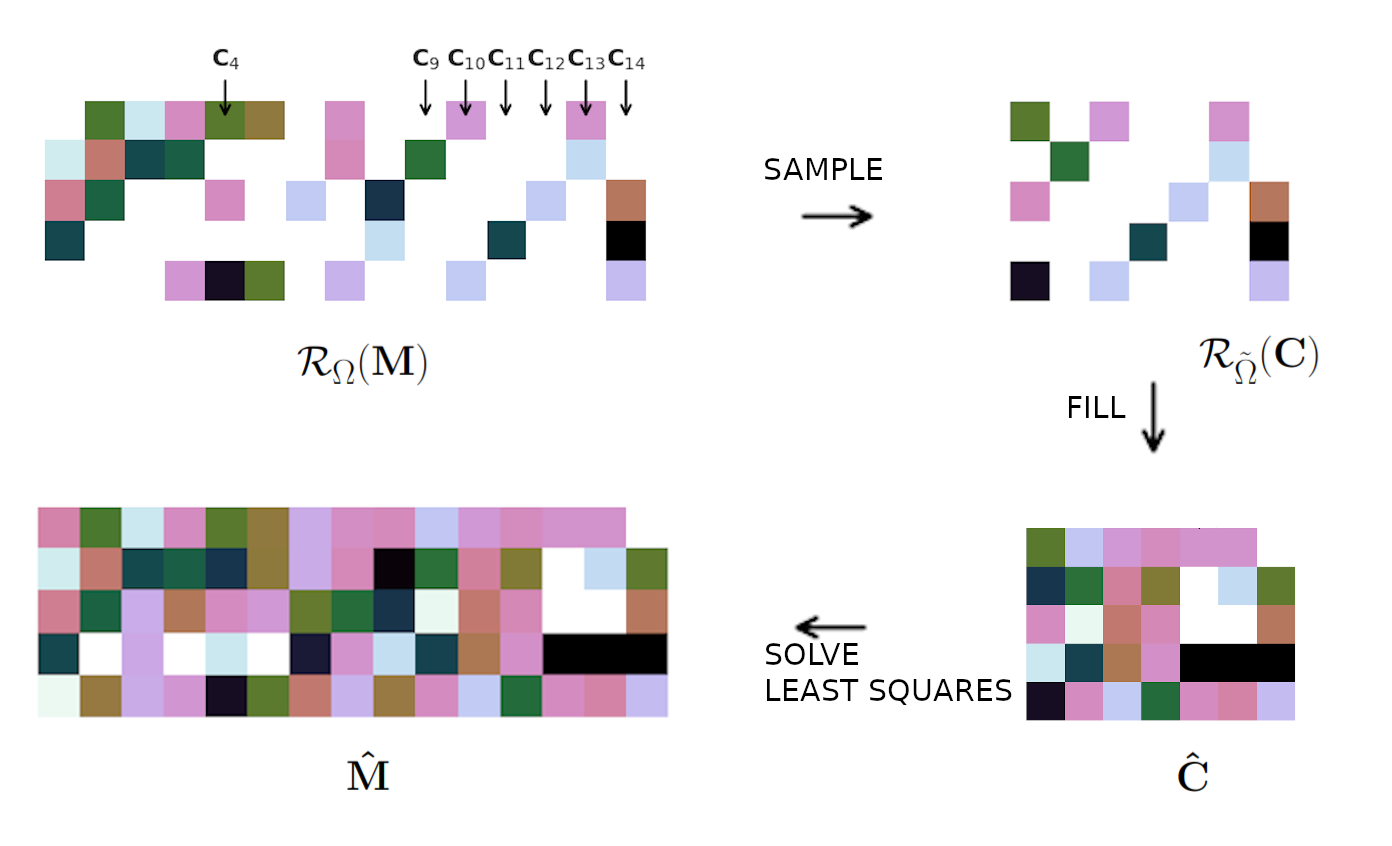}
        \caption[CSMC method.] {The CSMC method overview.}
        \label{fig:CSMC}
\end{figure}

\begin{algorithm}
\caption{CSMC-$\alpha$}
\begin{algorithmic}
\State Require: $\mathcal{R}_{\Omega}(\mathbf{M})$: $n_1 \times n_2$-matrix
\State Require: $\alpha$: Ratio of the selected columns
\State $I \leftarrow \text{Uniformly sample } \alpha \cdot n_2$ indices
\State $\mathbf{C}_{\text{missing}} \leftarrow \mathcal{R}_{\Omega}(\mathbf{M})_{: I}$ 
\State $\mathbf{\hat{C}} \leftarrow \text{MC}(\mathbf{C}_{\text{missing}})$ \Comment{Complete submatrix $\mathbf{C} \in \mathbb{R}^{n_1 \times d}$.}
\State $\mathbf{\hat{Z}} \leftarrow \text{arg} \min \limits_{\mathbf{Z} \in \mathbb{R}^{n_2}} \frac{1}{2} \|\mathcal{R}_{\Omega}(\mathbf{M}) - \mathcal{R}_{\Omega}(\mathbf{\hat{C}} \mathbf{Z}) \|_{F}^2$ 
\State $\mathbf{\hat{M}} \leftarrow \mathbf{\hat{C}} \mathbf{\hat{Z}}$
\State\Return $\mathbf{\hat{M}} $
\end{algorithmic}
\label{alg:csmc}
\end{algorithm}

\subsection{Theoretical results}

In their seminal work, Candes and Recht proved that, under certain assumptions, solving nuclear norm minimization leads to successful matrix recovery \cite{Candes08}. There has been a significant amount of research extending these results to noisy settings, improved bounds on the sample complexity and analyzing the robustness of the nuclear norm minimization algorithms to noise and outliers \cite{Recht09, Candes08, Moitra18}. Recht \cite{Recht09} greatly simplified the analysis. The theoretical guarantees rely on properties such as matrix rank, number and distribution of observed entries and the fact that the singular vectors of the matrix are uncorrelated with the standard basis. The last property is formalized with the matrix incoherence parameter and defined as follows  \cite{Candes08}.

\begin{definition}\label{def:coherence}
 Let $U$ be a subspace of $\mathbb{R}^n$ of dimension $r_U$ and $P_U$ be the orthogonal projection onto $U$. The incoherence parameter of $U$ is defined as  
 
\begin{align}
    \mu(U) = \frac{n}{r_U}\max_{1 \leq i\leq n }||P_U\mathbf{e}_i||_2^2,
\end{align}
\noindent
where $\mathbf{e}_i$ for $i= 1, \ldots n$ are the standard basis vectors of $\mathbb{R}^n$.

The incoherence parameter of the  rank-$\tilde{r}$  matrix $\mathbf{X}$ is given by
$\mu_0(\mathbf{X}) = \max \{ \mu(U), \mu(V)  \}$ where $U$ and $V$ are the linear spaces spanned by its singular vectors.
\end{definition}
Fig.~\ref{fig:three graphs} illustrates one-dimensional subspaces $U \subset \mathbb{R}^2$ with different incoherence parameter values. The incoherence parameter satisfies the bounds $1 \leq \mu(U) \leq \tfrac{n}{r_U}$ \cite{Chen18}. For a matrix $\mathbf{X}$, we say it has low incoherence if $\mu_0(\mathbf{X}) = \mathcal{O}(1)$, meaning that its singular vectors are spread out rather than concentrated on a few coordinates. In recommendation systems, this corresponds to variance being distributed across many users and items. Conversely, a high incoherence value $\mu_0(\mathbf{X})$ means that some singular vector is nearly aligned with a standard basis vector, so the variance is dominated by a single user (row) or a single item (column).

\begin{figure*}
     \centering
     \begin{subfigure}[b]{0.3\textwidth}
         \centering
         \includegraphics[width=\textwidth]{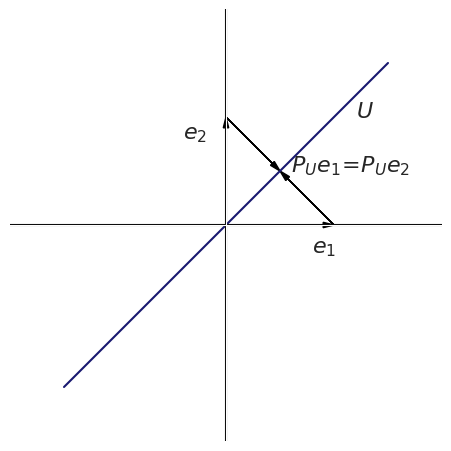}
         \caption{$\mu(U)=1$.}
         \label{fig:incoherent}
     \end{subfigure}
     \hfill
     \begin{subfigure}[b]{0.3\textwidth}
         \centering
         \includegraphics[width=\textwidth]{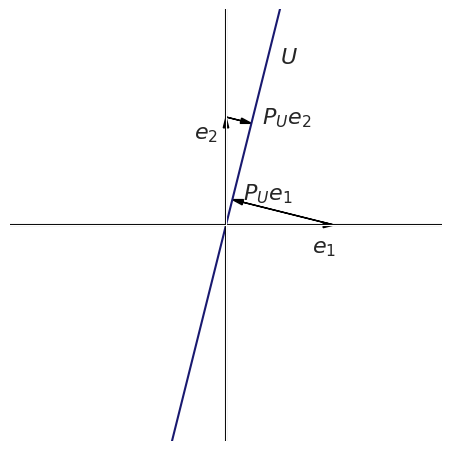}
         \caption{$1 < \mu(U) < 2$}
         \label{fig:medium_coherent}
     \end{subfigure}
     \hfill
     \begin{subfigure}[b]{0.3\textwidth}
         \centering
         \includegraphics[width=\textwidth]{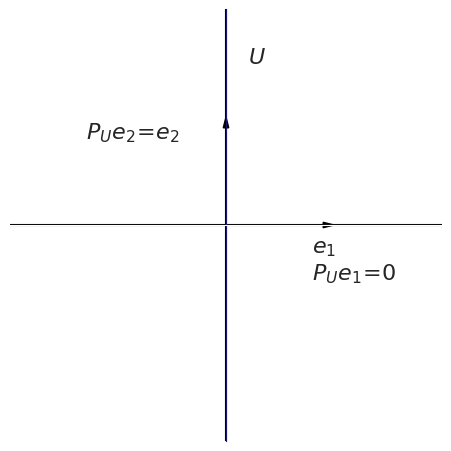}
         \caption{$\mu(U)=2$}
         \label{fig:coherent}
     \end{subfigure}
        \caption[The incoherence parameter of the one-dimensional subspace $U \subset \mathbb{R}^2$.] {Incoherence of the one-dimensional subspace $U \subset \mathbb{R}^2$: Fig. \ref{fig:incoherent} depicts $U$, spanned by $(\frac{1}{\sqrt{2}}, \frac{1}{\sqrt{2}})$, with the smallest incoherence parameter $\mu(U)=1$. Fig. \ref{fig:medium_coherent} illustrates the case where $1 < \mu(U) < 2$, while Fig. \ref{fig:coherent} presents $U$, spanned by one of the standard basis vectors $\mathbf{e}_2$, which achieves maximal incoherence parameter $\mu(U)=2$.}
        \label{fig:three graphs}
\end{figure*}

We assume that observed entries are sampled uniformly at random. We begin with a discussion on how the crucial characteristics of the matrix $\mathbf{M}$ completion problem are transferred to the task of filling in the submatrix $\mathbf{C}$.  Since $\mathbf{C}$ consists of randomly selected $d$ columns of $\mathbf{M}$, the rank of $\mathbf{C}$ is bounded by a minimum of $r$ and $d$, and the entries of $\mathbf{C}$ also follow a uniform distribution.  The following theorem demonstrates that for an incoherent, well-conditioned matrix $\mathbf{M}$, if the submatrix $\mathbf{C}$ is composed of $\mathcal{O}(r \ln(rn_2))$ sampled columns, it will exhibit low coherence.

\begin{theorem}[Corollary 3.6 in Cai et al. \cite{Cai21r}]\label{thm:cai}
Suppose that  $\mathbf{M} \in \mathbb{R}^{n_1 \times n_2}$ has rank $r$ and incoherence paramater bounded by $\mu_0(\mathbf{M})$. Suppose that $I \	\subseteq \{1, \ldots n_2 \}$ is chosen by sampling uniformly without replacement to yield $\mathbf{C} = \mathbf{M}_{:I}$. Let $d$ be the number of sampled columns such that $d \geq 1.06\mu_0(\mathbf{M})r \ln(rn_2)$. Then 
\begin{align}
\mu_0(\mathbf{C}) \leq 100 \kappa^2({\mathbf{M}}) \mu_0({\mathbf{M}}),
\end{align}
with a probability at least $1 - \frac{1}{n_2}$, where $\kappa(\mathbf{M})$ denotes the spectral condition number of $\mathbf{M}$, given by $\kappa(\mathbf{M}) := \| \mathbf{M} \|_2 \| \mathbf{M}^{\dagger} \|_2$.
\end{theorem}

In this section, the primary focus lies on assessing the recovery capabilities of the least-squares problem during the second stage of the CSMC. In particular, we assume that submatrix $\mathbf{C}$ is fully observed or perfectly recovered and provide assumptions that solving

\begin{align}\label{eq:cs_incomplete}
\underset{\mathbf{Z} \in \mathbb{R}^{d \times n_2}}{\mathrm{minimize}} \frac{1}{2} \|\mathcal{R}_{\Omega}(\mathbf{M}) - \mathcal{R}_{\Omega}(\mathbf{C} \mathbf{Z}) \|_{F}^2
\end{align}
\noindent
will output $\mathbf{\hat{Z}} \in \mathbb{R}^{d \times n_2}$, such that $\mathbf{M} = \mathbf{C}\mathbf{\hat{Z}}$ holds with high probability. While similar to the CUR+ framework introduced in \cite{Xu15}, our approach to solving (\ref{eq:cs_incomplete}) avoids the need to compute singular vectors of $\mathbf{C}$. The main result of this section is given by the Theorem \ref{thm:csmc_g}.  We believe that the following main result broadens the result of Xu et al. formulated as Theorem 2 in \cite{Xu15}. The proof of Theorem \ref{thm:csmc_g}, accompanied by the development of two additional theorems, which are our original contributions, are provided in Section V.

\begin{theorem}\label{thm:csmc_g}
Let $r$ be the rank of $\mathbf{M}\in \mathbb{R}^{n_1 \times n_2}$ and let $\mathbf{C} \in \mathbb{R}^{n_1 \times d}$ be a column submatrix of $\mathbf{M}$ formed by uniformly sampled without replacement $d$ columns. Let $\tilde{r}$ denote the rank of $\mathbf{C}$. Assume that for the parameter 
$\gamma > 0$,

\begin{enumerate}
\item $d \geq 7 \mu_0(\mathbf{M}) r(\gamma+\ln r)$,
\item $ |\Omega | \geq \tilde{r} n_2 \mu_0(\mathbf{C})\left(\gamma+\ln{(\frac{n_2\tilde{r}}{2})}\right)$.
\end{enumerate}
Let $\mathbf{\hat{Z}} \in \mathbb{R}^{d \times n_2}$ be the minimizer of the problem \ref{eq:cs_incomplete}. Then, with a probability at least $1-3e^{-\gamma}$, we have $\mathbf{M}=\mathbf{C}\mathbf{\hat{Z}}$.
\end{theorem}

\subsubsection{Discussion}

We will now discuss the theoretical aspects of CSMC method. To ensure a high probability of success, we set $\gamma = \Omega(\ln r)$ in Theorem~\ref{thm:csmc_g}. Thus, the number of sampled columns in the first stage order of $d = \Theta(\mu_0(\mathbf{M})r \ln r)$.

\paragraph{Matrix Size Reduction} Theorem~\ref{thm:csmc_g} implies that, under the assumption that $\mathbf{M}$ is incoherent, CSMC in its first stage can complete a much smaller matrix of size $n_1 \times \mathcal{O}(r \ln r)$ instead of applying a memory-intensive convex solver to the full $n_1 \times n_2$ matrix, thereby reducing memory overload and successfully recovering $\mathbf{M}$ by solving~(\ref{eq:randomized_nuclear_norm2}).

\paragraph{Sample Complexity} To successfully recover $\mathbf{M}$, CSMC first completes a column submatrix $\mathbf{C}\in\mathbb{R}^{n_1\times d}$. The nuclear-norm minimization subroutine requires $\mathcal{O}\!\big(\mu_0(\mathbf{C})(n_1+d)\tilde r\ln^2(n_1+d)\big)$ uniformly sampled entries in Stage~I \cite{Recht09, Moitra18, Chen18}. Stage~II requires $\mathcal{O}\!\big(n_2\tilde r\,\mu_0(\mathbf{C})\ln(n_2r)\big)$ entries by Theorem~\ref{thm:csmc_g} for $\gamma = \Omega(\ln r)$. By Theorem~\ref{thm:cai}, for well-conditioned matrices we have $\mu_0(\mathbf{C})=\mu_0(\mathbf{M})$. Together with the trivial bounds $\tilde{r}\le r$ and $d \le n_2$, this yields the overall sample complexity of CSMC 
\[ 
\mathcal{O}\!\big(\mu_0(\mathbf{M})\,r\,(n_1+n_2)\ln^2(n_1+n_2)\big),
\]
which matches the standard sample complexity of nuclear-norm minimization. Sharper bounds are possible when setting the theorem-prescribed 
$d=\Theta(r\ln r)$. In this case, and provided that $n_2 \gg (n_1+r\ln r)\ln n_2$, the total sample complexity reduces to
\[
\mathcal{O}\!\big(\mu_0(\mathbf{M})\,r\big((n_1+r\ln r)\ln^2 n_2 + n_2\ln n_2\big)\big),
\]
so that the $n_2$-term dominates and yields an improvement of a factor $\Theta(\ln n_2)$ compared to the nuclear-norm minimization.

Finally, we note that Stage~II of CSMC can also be interpreted as a fast standalone matrix completion method under a sampling scheme in which the $d$ columns are fully observed, and an additional subset $\Omega$ of entries is sampled uniformly at random. Under these assumptions, Theorem~\ref{thm:csmc_g} guarantees exact recovery of $\mathbf{M}$ with 
\[
d n_1 + |\Omega| = \mathcal{O}\!\big(\mu_0(\mathbf{M}) r(n_1+n_2)\ln(n_2r)\big)
\]
observations when $\gamma=\Omega(\ln r)$ and $\mu_0(\mathbf{C}) = \mathcal{O}(\mu_0(\mathbf{M}))$. Table~\ref{tab:sample_complexity} compares the sample complexity of CSMC with the sample complexities of other matrix completion algorithms that use CSS.

\begin{table*}[h]
\caption{Sample complexities of different MC algorithms $n=n_1+n_2$}
\begin{adjustbox}{width=\textwidth}
\begin{tabular}{@{}llllll@{}}
\toprule
Algorithm & CSMC & CUR+ \cite{Xu15} & MC2 \cite{Ward18} & AMC \cite{Krishnamurthy14} \\
\midrule
Sample complexity &$\mathcal{O}(\mu_0(\mathbf{M})rn\ln^2n)$ & $\mathcal{O}(\mu^2_0(\mathbf{M})rn\ln r)$. & $\mathcal{O}(\mu_0(\mathbf{M})rn \log^2 n)$ &  $\mathcal{O}( \mu(U)rn\ln^2 r)$ \\
Full-column sampling & No  & Yes & No & Yes  \\
\end{tabular}
\end{adjustbox}
\label{tab:sample_complexity}
\end{table*}

\paragraph{Choice of sampling columns algorithm} To develop a fast and efficient matrix-completion method, we adopt a uniform column sampling strategy, which is computationally inexpensive and does not require access to all matrix entries. Moreover, Xu et al.~\cite{Xu15} showed that under the standard incoherence assumption, uniformly sampled columns accurately approximate the column space of a matrix. We acknowledge that uniform sampling does not capture differences between columns. Alternative strategies, such as leverage score sampling \cite{hamm2020stability, Drineas08}, volume sampling \cite{Deshpande10}, or determinantal point process sampling \cite{belhadji2020determinantal}, may yield better approximations but typically require full access to matrix entries. On the other hand, adaptive sampling methods, while potentially effective for matrices with high incoherence parameters, incur additional computational overhead \cite{Wang17, Krishnamurthy14, Krishnamurthy2013}.

\paragraph{Robustness of CSMC} Our analysis so far has focused on exact matrix completion in the noiseless setting. In practice, however, the two-stage structure of CSMC raises two important considerations. First, errors from Stage I may propagate to Stage II. While nuclear norm minimization can recover $\mathbf{M}$ with high probability, the required number of observations depends on the unknown rank and incoherence parameters, so small Stage I errors may occur. To mitigate their impact, future extensions of CSMC may incorporate weighted least squares or additional regularization in Stage II to reduce sensitivity to imperfect inputs. Second, robustness to noise must be ensured in both stages. For Stage I, robustness can be achieved by relaxing equality constraints in (\ref{eq:exact_nuclear_norm}) to inequalities or by solving a regularized formulation such as (\ref{eq:lagrange}), both of which enjoy theoretical guarantees \cite{Mazumder2010a}. For Stage II, robustness can be introduced by regularizing problem~(\ref{eq:randomized_nuclear_norm2}), for example, by penalizing the Frobenius norm of $\mathbf{Z}$, or by adopting RPCA-inspired formulations to capture sparse corruptions via an $\ell_1$ penalty. These modifications would enable CSMC to handle both dense small-magnitude and sparse large-magnitude corruptions.

\subsection{Implementation}
The CSMC is a versatile method for solving matrix completion problems. It allows the incorporation of various matrix completion algorithms in the first step and least squares algorithms in the second step. This paper discusses two algorithms that implement the CSMC method, each tailored to address matrix completion problems of different sizes.

\subsubsection{Columns Selected Nuclear Norm (CSNN)} A submatrix is filled with the exact nuclear norm minimization using SDP solver. The CSNN is dedicated to recovering the  small and medium size $\mathbf{M}$. We use the first-order Splitting Conic Solver(SCS)~\cite{scs, Donoghue21}, which offers better scalability for large problems compared to interior-point methods that require costly Hessian computations. Although SCS reduces memory and per-iteration computational costs, it converges more slowly and is still unsuitable for very large-scale matrices. 
  
  We found the first-order Splitting Conic Solver (SCS) as an efficient way to solve the SDP \cite{Donoghue21} in the Stage I. To find the optimum of (\ref{eq:randomized_nuclear_norm2}), we solve an ordinary least squares problem for each column, restricted to its observed entries. This results in $n_2$ independent problems, where each involves only the available data for one column, with an average of $\frac{|\Omega|}{n_2}$ observed entries. This formulation significantly reduces computational complexity and naturally enables efficient distributed implementations.
  
  \subsubsection{Columns Selected Proximal Gradient Descent (CSPGD)} Inexact nuclear norm minimization (\ref{eq:lagrange}) is solved by Proximal Gradient Descent (PGD). PGD is an efficient first-order method for convex optimization. It leverages the closed-form expression for the proximal operator for the objective function (\ref{eq:lagrange})  and enables enables the completion of large matrices with a convergence rate of $\mathcal{O}(1/t)$, where $t$ is the iteration number \cite{Mazumder2010a}. However, it requires calculation of the Singular Value Thresholding (SVT) algorithm, which relies on SVD in each iteration. This computation becomes expensive for large matrices. CSPGD benefits CSS strategy, which improves efficiency by reducing the dimensionality of the problem. To find the optimum of the regression problem (\ref{eq:randomized_nuclear_norm2}), we directly solve the least squares for the observed entries in each column.

We developed open-source code implementing  the CSNN and CSPGD methods. Algorithms support both Numpy arrays \cite{Harris20} and PyTorch tensors \cite{Torch}, with the latter offering the advantage of GPU acceleration for. The SDP is solved using the Splitting Conic Solver (SCS) \cite{scs}. The code is written in Python 3.10 and is available at \mbox{\url{https://github.com/ZAL-NASK/CSMC}}.

\section{Numerical experiments}

All numerical experiments were implemented in Python 3.10. For benchmarking, the Matrix Factorization (MF) method was taken from the fancyimpute library \cite{fancyimpute}.
In addition to the proposed algorithm, we implemented several state-of-the-art non-convex matrix completion methods, including SVP \cite{jain10, jain14} and ScaledGD \cite{tong2021accelerating}, as well as algorithms that incorporate CSS principles discussed earlier, namely MC2 \cite{Ward18} and CUR+Nuc, which is based on CUR+ \cite{Xu15}. Experiments were executed on a Linux workstation equipped with an 11th Gen Intel(R) Core(TM) i7-1165G7 @ 2.80 GHz (8 cores) and 32 GB RAM. Experiments S~II, S~III, and the Jester Joke Recommendation System were executed on a Linux server equipped with an AMD EPYC Processor (18 cores) and 256 GB RAM under KVM virtualization.

\subsection{Synthetic data set}
To evaluate the proposed CSMC method, we conducted experiments in a controlled synthetic environment and compared its performance against the above baselines. The goal of each experiment was to recover a random $n_1 \times n_2$ matrix $\mathbf{M}$ with the ratio of the observed entries $\rho$. To control the rank of $\mathbf{M}$ equal to $r$, the test matrix was generated as a product of the $n_1 \times r$ matrix $\mathbf{A}$ and $r \times n_2$ matrix $\mathbf{B}$. Matrices $\mathbf{A}$ and $\mathbf{B}$ were generated in two steps. In the first one, matrix entries were sampled from the normal distribution $\mathcal{N}(0, 1)$. In the second step, noise matrices with the ratio $0.3$ of non-zero entries were added to each matrix. 

We evaluated the result quality using the relative approximation error
\begin{align}\label{eq:epsilon}
\epsilon := \frac{ \|\mathbf{M}-\mathbf{\hat{M}}\|_F}{\| \mathbf{M}\|_F},
\end{align}
and compared the runtimes of all algorithms. As CSMC involves a two-stage process, unlike the single-stage approach of other algorithms, we did not compare the iteration counts required to achieve a given solution quality.

\subsubsection{Experiment S I: Sample Complexity}

The first experiment aimed to assess how the solution quality varies with the number of sampled columns ($d =  \lfloor \alpha n_2\rfloor $) and the observed-entry ratio $\rho$.

\paragraph{Settings}
We ran the first experiment on a set of $300 \times 1000$ matrices with ranks 5 and 10. We varied the ratio of known entries $\rho$ and the number of sampled columns $d$. Each configuration was executed over $N_{\text{trial}} = 20$ independent runs. The performance of the algorithms NN and CSNN-$\alpha$ ($\alpha \in {0.1, \ldots, 0.9}$) was evaluated across all trials. To illustrate the proportion of runs achieving a given error $\epsilon$ (Eq.~\ref{eq:epsilon}), we computed the empirical cumulative distribution function (ECDF) \cite{pawel19}. The ECDF is defined as $\hat{F}S: \mathbb{R}{+} \rightarrow [0, 1]$,

\begin{align}\label{eq:ecdf}
\hat{F}_S(a) = \frac{|{ s \mid \epsilon_s \leq a }|}{|S|},
\end{align}

\noindent
where each trial is represented by $s$, $\epsilon_s$ denotes the relative error obtained in trial $s$, and $S$ is the set of all trials for a given parameter configuration.  

\paragraph{Results}

The experiment highlighted that reducing column sampling in CSNN improves runtime but can affect recovery success. However Fig. \ref{fig:si_ecdf_approx} shows that sampling CSNN-$0.2$ offered a tenfold (10x) increase in time savings compared to the NN algorithm and maintains a solution quality. For rank-5 matrices, all tested $\alpha$ achieved successful recovery when $\rho \ge 0.3$, while rank-10 matrices required $\rho \ge 0.4$. For 80\% missing entries, CSNN-0.2 and CSNN-0.3 produced good solutions for ranks 5 and 10, respectively, demonstrating that a small ratio of sampled columns can substantially reduce runtime while maintaining high reconstruction quality.

\begin{figure*}[htb]
\centering
\begin{subfigure}[b]{\textwidth}
    \includegraphics[width=\textwidth]{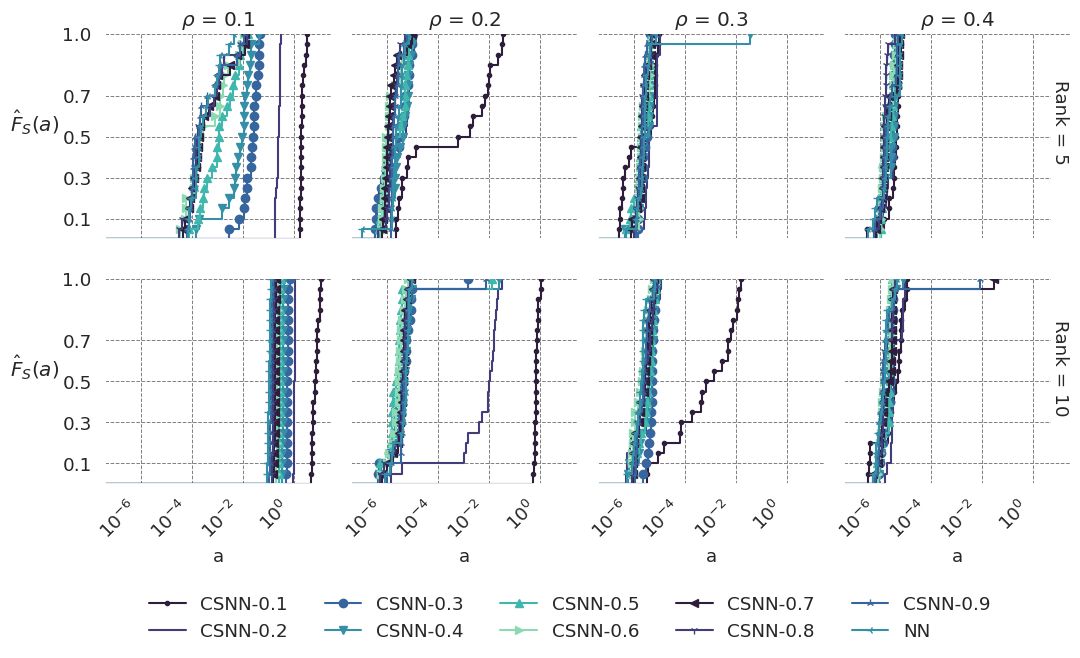}
\end{subfigure}
\begin{subfigure}[b]{\textwidth}
\includegraphics[width=\textwidth]{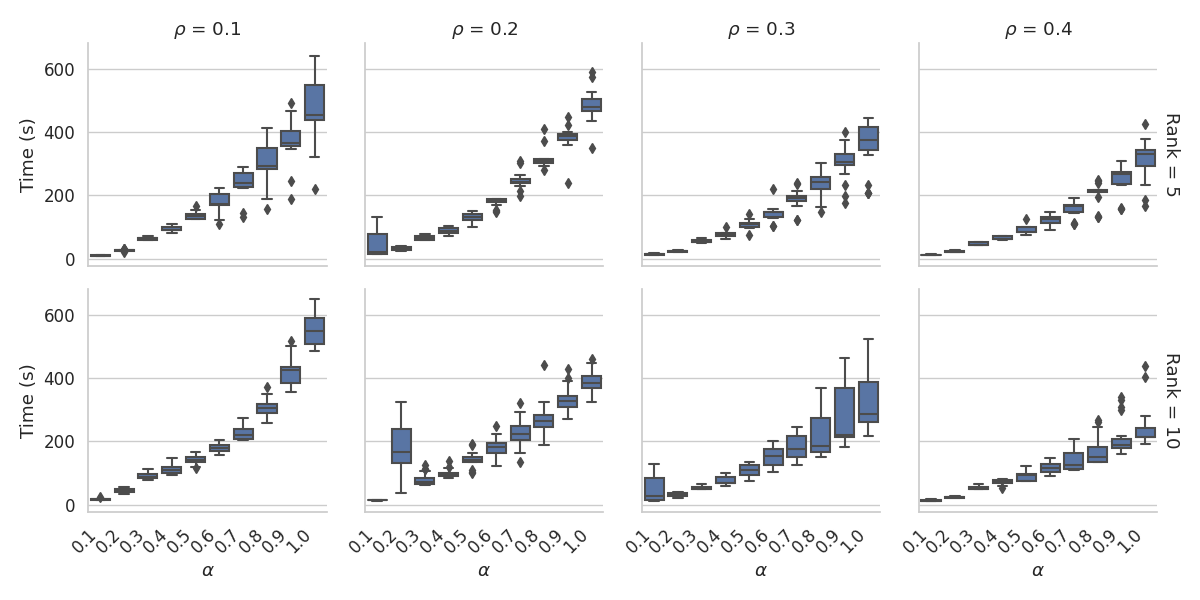}
\end{subfigure}
\caption[Results S I - relative error ECDF]{Results S I: ECDF (\ref{eq:ecdf}) and runtimes  depending on the matrix rank and rate of the known entries for NN and CSNN-$\alpha$,  $\alpha \in \{0.1,\cdots, 0.9\}$, $\mathbf{M} \in \mathbb{R}^{300 \times 1000}$.}
\label{fig:si_ecdf_approx}
\end{figure*} 

\subsubsection{Experiment S II: Comparison with non-convex methods} 

To evaluate the scalability and efficiency of the proposed CSNN method relative to state-of-the-art fast non-convex algorithms, we compared it with SVP \cite{jain10, jain14}, ScaledGD \cite{tong2021accelerating}, and Matrix Factorization solved via stochastic gradient descent (MF).

\paragraph{Settings}
We benchmarked the CSNN model, using a fixed number of sampled columns ($d=400$), against SVP, ScaledGD, and MF. For the non-convex algorithms, we tested various rank parameter $k$ values to assess their sensitivity to rank selection (\ref{eq:burer_monteiro}). Each algorithm was applied to the completion of rank-5 matrices with dimensions $n_1 = 400$ and $n_2 \in \{1000, 5000, 10000, 50000, 75000\}$ under an observed-entry ratio of $\rho = 0.2$. Performance was evaluated in terms of the approximation error (\ref{eq:epsilon}) and runtime to assess both reconstruction quality and computational scalability. Each configuration was executed over $N_{\text{trial}} = 10$ independent runs.

\paragraph{Results}
CSNN consistently achieved a relative approximation error on the order of $10^{-6}$ while maintaining a stable runtime of approximately $47$~s across all tested matrix sizes (Fig.~\ref{fig:scaling}). In contrast, MF exhibited lower accuracy ($\epsilon \approx 10^{-3}$) and substantially longer runtimes. Both ScaledGD and SVP were highly sensitive to the choice of the rank parameter~$k$. When $k$ matched the true rank of~$\mathbf{M}$ ($k=r$), their errors reached $10^{-7}$--$10^{-10}$, whereas even a slight overestimation of $r$ increased the error to the order of $10^{-2}$. Moreover, the runtime of ScaledGD and SVP increased rapidly with $n_2$, reaching several thousand seconds for the largest matrices. In contrast, CSNN maintained nearly constant runtime, demonstrating superior scalability.

\begin{figure}[!hb]
\centering
    \includegraphics[width=\textwidth]{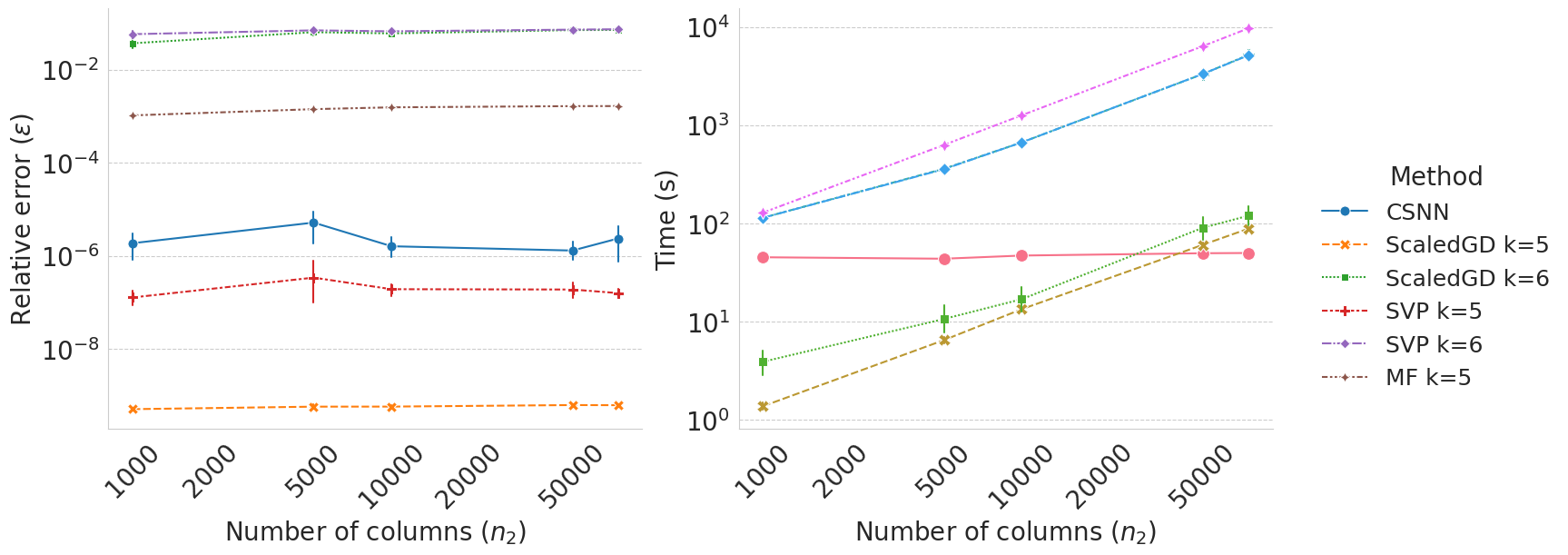}
\caption[Results S II]{Results S II: relative approximation error and runtime.}
\label{fig:scaling}
\end{figure} 

\subsubsection{Experiment S III: Comparison with CSS approaches}
We compared CSNN-0.3 with NN, MC2 (which estimates leverage scores of matrix entries) \cite{Ward18}, and CUR+Nuc, a two-stage CUR+ based matrix completion method. In the first stage of CUR+Nuc, randomly sampled rows and columns are completed using nuclear norm minimization, and in the second stage, a regression problem is solved on the filled submatrices, as described in \cite{Xu15}. Similar to CUR+, CUR+Nuc depends on the parameter~$k$, which specifies the number of singular vectors of the row and column submatrices used in the regression step. We did not include CUR+ \cite{Xu15} and AMC \cite{Krishnamurthy14} in this comparison, as both focus on different sampling patterns—specifically, they assume that entire columns of $\mathbf{M}$.

\paragraph{Settings}
To assess the influence of the rank parameter $k$ in CUR+Nuc, we conducted experiments with $k=9$ and $k=10$ and analyzed its effect on performance. The comparison was performed on matrices of size $500 \times 1000$ with rank 10, with 20\% of entries known, over 20 trials.

\paragraph{Results}

Fig.~\ref{figi:siii} compares CSNN-0.3 with NN, MC2, and CUR+Nuc. CSNN-0.3 achieved the same error as NN and MC2 ($10^{-6}$), indicating that sampling based on estimated leverage scores did not improve solution quality compared to uniform sampling. CUR+Nuc was highly sensitive to the choice of the rank parameter~$k$: when $k$ was set to one less than the true rank ($r=10$, $k=9$), the approximation error increased substantially, whereas setting $k$ to the true rank ($r=k=10$) resulted in accuracy comparable to the other algorithms. In terms of runtime, CSNN-0.3 was approximately five times faster than NN and MC2, and more than 3.5 times faster than CUR+Nuc, demonstrating its clear advantage in computational efficiency for MC tasks.

\begin{figure}[!hb]
\centering
    \includegraphics[width=\textwidth]{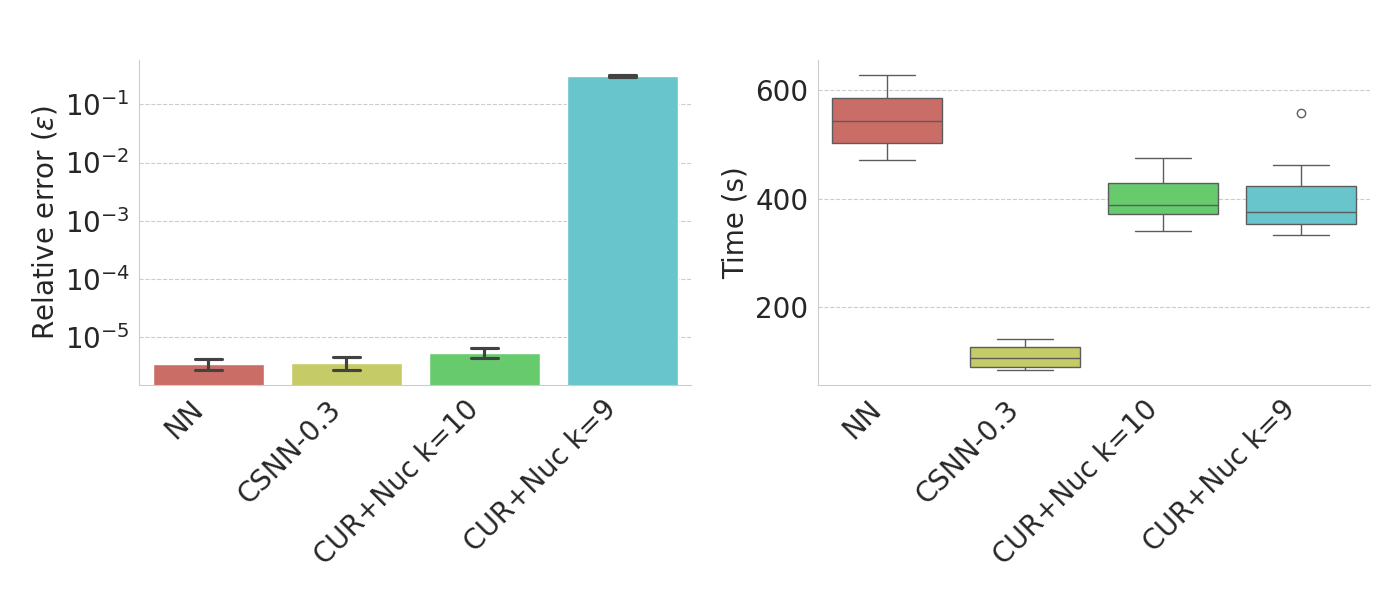}
\caption[Results S III]{Results S III: relative approximation error and runtime.}
\label{figi:siii}
\end{figure}

\subsection{Recommendation system}
Here, we assess the performance of CSNN, CSPGD, and the state-of-the-art methods SVP and ScaledGD in the context of recommendation systems, where matrix completion serves as a rating prediction technique. The evaluation was conducted on two diverse datasets: the MovieLens dataset \cite{Harper15}, which contains discrete ratings from~1~to~5, and the Jester Collaborative Filtering dataset \cite{jesterGoldberg01}, which includes continuous ratings in the range from~$-10{,}000$~to~$10{,}000$ provided by~24{,}983 users who rated~36~or more jokes. 

\subsubsection{Movie Lens Dataset} The goal of the first experiment was to empirically verify that CSMC maintains the accuracy of nuclear-norm-based MC while improving computational efficiency.

\paragraph{Settings}
To benchmark CSNN and CSPGD algorithms, we constructed two data sets: Movie Lens Small and Movie Lens Big. The Movie Lens Small was represented by $\mathbf{M} \in \mathbb{R}^{140 \times 668}$ matrix obtained from the Movie Lens Small dataset, which contained 100 000 5-star ratings applied to 9742 movies by 610 users. Since the original matrix was too large for SDP solvers, we sorted the data by user frequency rate and took data containing rates made by the top 60\% users. Then, we sorted the obtained data by movie frequency rate and selected the top 50\% movies. The obtained matrix had $\rho=0.25$ known entries. We evaluated CSPGD algorithms on the Movielens 25M data set, containing 25 million ratings applied to 62,000 movies by 162,000 users. Again, due to the large size and low observation rate, we extracted a $654 \times 27813$ matrix $\mathbf{M}$. The known entries rate $\rho$ was equal to 0.09. 

We assessed the performance of the NN and CSNN-$\alpha$ for $\alpha \! \in \! \{ 0.3, \! 0.4, \! 0.5, \! 0.7\}$ and CSPGD-$\alpha$ for $\alpha \in \{ 0.3, 0.5\}$. We followed previous work \cite{Cai22cur} and employed the Cross-Validation method.  In each trial of the experiment, $\Omega$ set was randomly split into training and testing sets denoted by $\Omega_{\text{train}}$ and $\Omega_{\text{test}}$. Specifically, we randomly selected $\rho$ rate of the observed set, and by assigning the null values to the rest of the entries, we constructed $\mathcal{R}_{\Omega_{\text{train}}}(\mathbf{M})$ matrix. We evaluated each algorithm on the $\Omega_{\text{test}}$ set. We conducted 20 independent experimental trials under each scenario. We compared Normalized Mean Absolute Error (NMAE)

\begin{align}\label{eq:nmae}
    \text{NMAE} = \frac{1}{| \Omega_{\text{test}}|(m_\text{max} -m_{\text{min}} )}\sum_{(i, j) \in \Omega_{\text{test}}} | \hat{m}_{ij} - m_{ij}|,
\end{align}
\noindent
where $m_\text{max}$ and $m_{\text{min}}$ denote the maximum and minimum rating, respectively, and $\Omega_{\text{test}}$ denotes the set of indices in test set. This metric is widely used to assess collaborative filtering tasks \cite{Cai22cur}. The quality of the recommendation was also measured as hit-rate, defined as 
\begin{align}
    \text{HR} = \frac{\# hits}{|  \Omega_{\text{test}} |},
\end{align}
\noindent
where a predicted rating was a \textit{hit} if its rounded value equals the actual rating. 

\paragraph{Results}

Fig. \ref{fig:mi} shows that the predictions of CSNN-$0.7$ maintained the quality of the NN algorithm in terms of NMAE: (0.16 vs 0.13) and HR (0.23 vs 0.25), and resulted in runtime savings from 57 seconds to 40 seconds (Table \ref{tab:mi}). CSPGD algorithms were much faster than PGD (Fig. \ref{fig:mii}). The HR values for PGD and CSPGD-$0.3$ were equal to 0.28 and 0.24, respectively (Table \ref{tab:mii}). 

\begin{centering}
\begin{table*}[!htb]
\caption[Results M I]{Results on Movie Lens Small dataset}
\centering
\begin{adjustbox}{width=0.7\textwidth}
\begin{tabular}{lrrrrrr}
\toprule
{} & \multicolumn{2}{c}{NMAE} &  \multicolumn{2}{c}{HR} & \multicolumn{2}{c}{Time [s]} \\
\cmidrule(lr){2-3} \cmidrule(lr){4-5} \cmidrule(lr){6-7} 
{} &       mean &   std &  mean &   std &         mean &   std \\
Algorithm &            &       &       &       &              &       \\
\midrule
CSNN-0.3  &      0.271 & 0.028 & 0.172 & 0.005 &       10.137 & 1.118 \\
CSNN-0.4  &      0.240 & 0.060 & 0.191 & 0.005 &       17.677 & 2.005 \\
CSNN-0.5  &      0.248 & 0.176 & 0.208 & 0.006 &       28.673 & 2.419 \\
CSNN-0.7  &      \textbf{0.156} & \textbf{0.011} & \textbf{0.232} & \textbf{0.004} &  \textbf{39.893} & \textbf{8.595} \\
NN        & \textbf{0.127} & \textbf{0.001} & \textbf{0.255} & \textbf{0.006} &  \textbf{56.924} & \textbf{2.435} \\
\bottomrule
\end{tabular}
\end{adjustbox}
\label{tab:mi}
\end{table*}
\end{centering}

\begin{figure*}[htb]
     \centering
     \includegraphics[width=\textwidth]{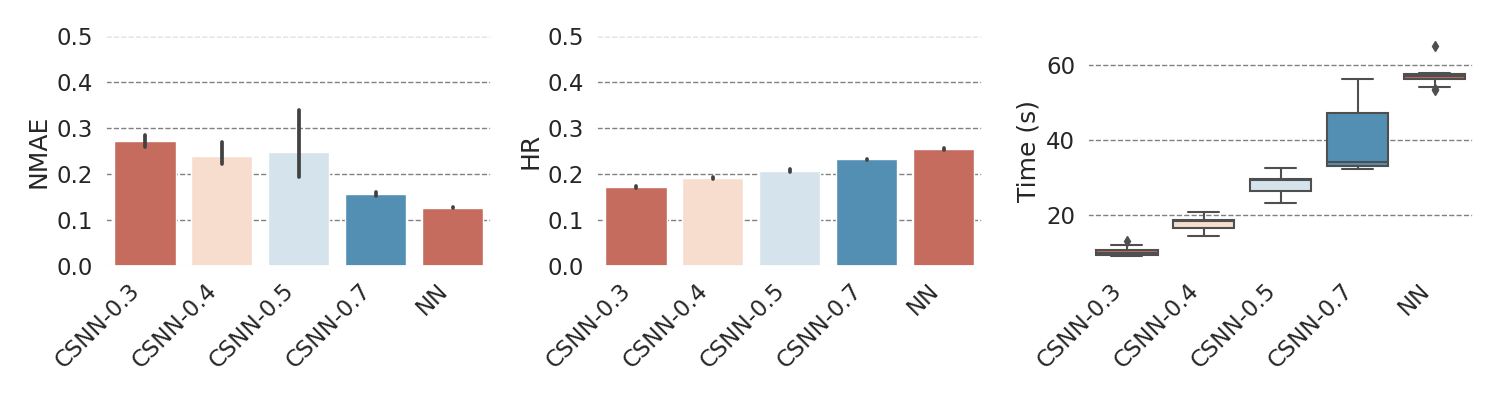}   
     \caption[Results M I]{Results on Movie Lens Small dataset;  NN and CSNN algorithms.}
     \label{fig:mi}
\end{figure*}  

\begin{table*}[!htb]
\caption[Results M II]{Results on Movie Lens Big dataset}
\centering
\begin{adjustbox}{width=0.7\textwidth}
\begin{tabular}{lrrrrrr}
\toprule
{} & \multicolumn{2}{c}{NMAE} & \multicolumn{2}{c}{HR} & \multicolumn{2}{c}{Time [s]} \\
\cmidrule(lr){2-3} \cmidrule(lr){4-5} \cmidrule(lr){6-7} 
{} &       mean &   std &  mean &   std &         mean &   std \\
Algorithm &            &       &       &       &              &       \\
\midrule
CSPGD-0.3  &    \textbf{0.138} & \textbf{0.001} & \textbf{0.245} & \textbf{0.002} & \textbf{23.032} & \textbf{0.321} \\
CSPGD-0.5  &      0.135 & 0.001 & 0.252 & 0.001 &       27.774 & 0.299 \\
PGD        &  \textbf{0.119} & \textbf{0.001} & \textbf{0.281} & \textbf{0.001} &  \textbf{449.026} & \textbf{9.502} \\
\bottomrule
\end{tabular}
\end{adjustbox}
\label{tab:mii}
\end{table*}

\begin{figure*}[!htb]
     \centering
     \includegraphics[width=\textwidth]{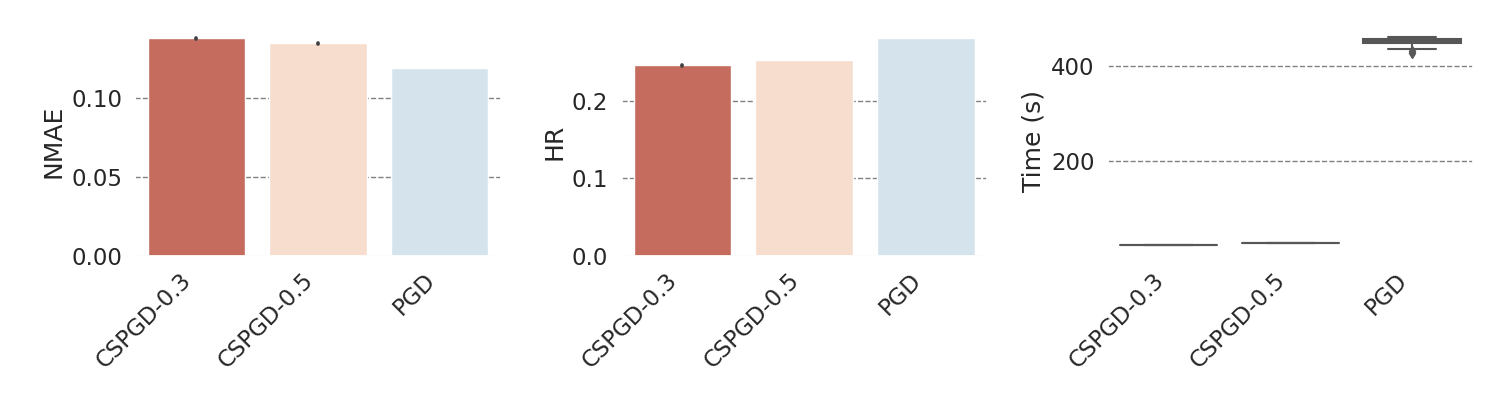}   
     \caption[Results M II]{Results on Movie Lens Big dataset; PGD and CSPGD algorithms.}
     \label{fig:mii}
\end{figure*} 
\newpage
\subsubsection{Jester Joke Recommendation System}
The goal of this experiment was to compare CSNN with state-of-the-art non-convex algorithms SVP and ScaledGD that are used in practice.

\paragraph{Settings}
We compared the performance of \textsc{CSNN}-$\alpha$ for $\alpha \in \{0.03, 0.04, 0.16\}$ against the state-of-the-art algorithms SVP and ScaledGD, each evaluated with parameters $k \in \{5, 10, 20,50\}$. 
The experiments were conducted on the Jester~Data~1 dataset, which contains ratings from 24{,}983 users who have rated at least 36 jokes, forming a matrix of size $101 \times 24{,}983$ \cite{jesterGoldberg01}. The dataset includes approximately $72\%$ observed entries.  For each configuration, an additional $80\%$ of the observed entries were randomly withheld to form a test set, leaving about $14\%$ of the entries available for training. All experiments were repeated over 10 trials.

\paragraph{Results} 

The ScaledGD and SVP methods achieved the best NMAE (0.24) when the $k$ parameter was set to 5 (Fig.~\ref{fig:jester}). However, they were both slower than CSNN-0.02, which required only 86 seconds to reach a smaller NMAE of 0.21. Moreover, CSNN-0.04 and CSNN-0.16 achieved better results, with NMAE values of 0.20 and 0.19, respectively, although their runtimes were longer—173 seconds and 4795.6 seconds (Table~\ref{tab:jester}).

\begin{figure}[!htb]
\centering
    \includegraphics[width=\textwidth]{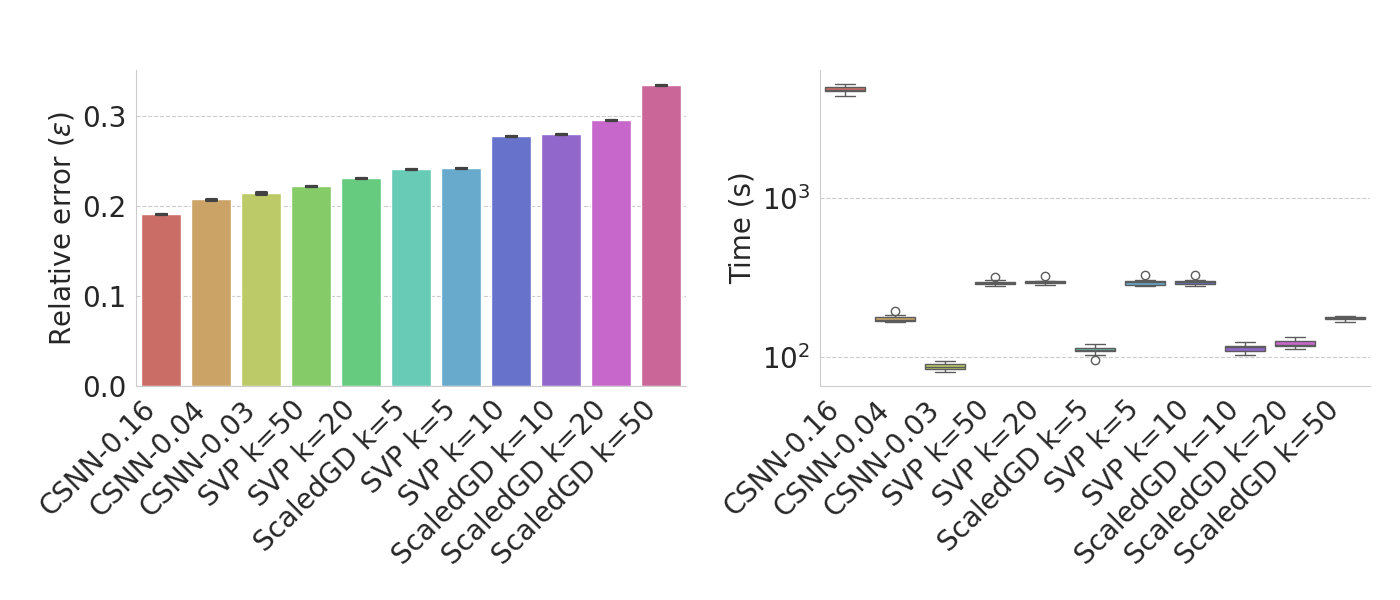}
\caption[Results S III]{Results for Jester Joke Recommendation System.}
\label{fig:jester}
\end{figure} 

\begin{centering}
\begin{table*}[!htb]
\caption[Results Synthetic]{Jester Joke Recommendation System results}
\centering
\begin{adjustbox}{width=0.6\textwidth}
\begin{tabular}{lrrrr}
\toprule
{} & \multicolumn{2}{c}{Error} & \multicolumn{2}{c}{Time [s]} \\
\cmidrule(lr){2-3} \cmidrule(lr){4-5}
{} & mean & std & mean & std \\
\midrule
CSNN-0.03 & \textbf{0.214} & \textbf{0.001} & \textbf{86.834} & \textbf{4.397} \\
CSNN-0.04 & 0.207 & 0.001 & 173.432 & 8.833 \\
CSNN-0.16 & \textbf{0.191} & \textbf{0.001} & \textbf{4795.6} & \textbf{237.312} \\
ScaledGD k=5 & 0.24 & 0.001 & 109.694 & 7.121 \\
SVP k=5 & 0.242 & 0.001 & 293.435 & 14.375 \\
\bottomrule
\end{tabular}
\end{adjustbox}
\label{tab:jester}
\end{table*}
\end{centering}
\newpage
\subsection{Image recovery}
Image inpainting, a technique in image processing and computer vision, fills in missing or corrupted parts. Additionally, estimating random missing pixels can be treated as a denoising method or utilized to accelerate rendering processes. Low-rank models are highly effective for this task, leveraging the assumption that images typically exhibit low-rank structures. The primary information of the image matrix is dominated by its largest singular values while setting the smallest singular values to zero can be done without losing essential details. In this paper, we compare nuclear norm minimization with SDP (NN) to our proposed CSNN method. Detailed experiments with PGD and its column-sampling variant (CSPGD), were reported in our conference paper~\cite{Krajewska24}. To avoid redundancy, we focus here on the core NN baseline while referring the reader to~\cite{Krajewska24} for comparisons with these variants.

\paragraph{Settings} The data set contained ten grey-scaled pictures of bridges downloaded from the public repository \footnote{image source (\url{https://pxhere.com/pl/})}. Images were represented as  $240 \times 360$ matrices. We assessed the performance of the algorithms among 100 independent trials (ten trials per picture). The quality of the reconstructed image was assessed with the signal-to-noise ratio (SNR)

\begin{align}\label{eq:snr}
    \text{SNR} (\mathbf{\hat{M}}) = 20 \log_{10} \left(\frac{\| \mathbf{M} \|_F}{\| \mathbf{\hat{M}} - \mathbf{M}\|_F}\right)
\end{align}
\noindent
and the relative error (\ref{eq:epsilon}).

\paragraph{Results} The CSNN-$0.7$ solutions maintained the quality of the NN in terms of SNR (Table \ref{tab:pi}). As shown in Fig. \ref{fig:pi_stats}, CSNN-$0.5$ offered considerable time savings and good relative error. Fig. \ref{fig:pi} presents one of the completed pictures.

    \begin{figure*}[!htb]

        \centering
        \begin{subfigure}[b]{0.495\textwidth}
            \centering
            \includegraphics[width=\textwidth]{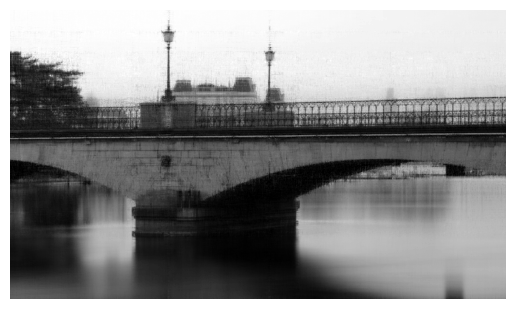}
            {Image restored with NN}    
        \end{subfigure}
        \begin{subfigure}[b]{0.495\textwidth}  
            \centering 
            \includegraphics[width=\textwidth]{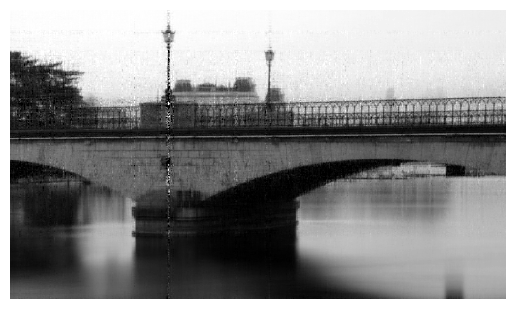}
            {Image restored with CSNN-0.7 }  
        \end{subfigure}
        \caption[Result P I]
        {Image inpainting; NN and CSNN-0.7 algorithms, $\rho=20\%$ known entries.} 
        \label{fig:pi}
    \end{figure*}

\begin{figure*}[!htb]
     \centering
   \begin{subfigure}[b]{0.6\textwidth} 
        \includegraphics[width=\textwidth]{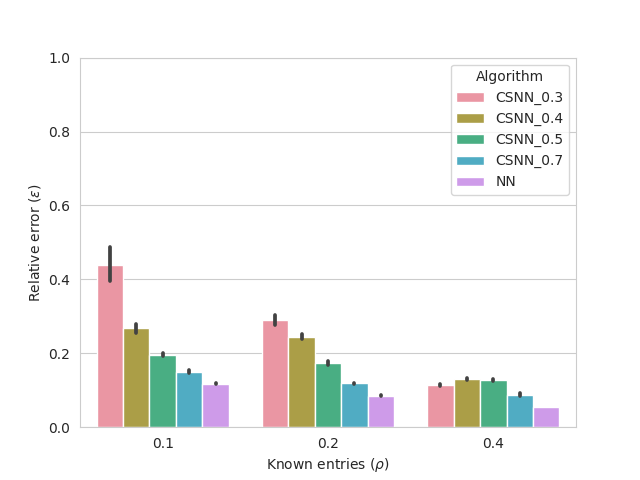}
    \end{subfigure}
    \hfill
    \begin{subfigure}[b]{0.6\textwidth}
        \includegraphics[width=\textwidth]{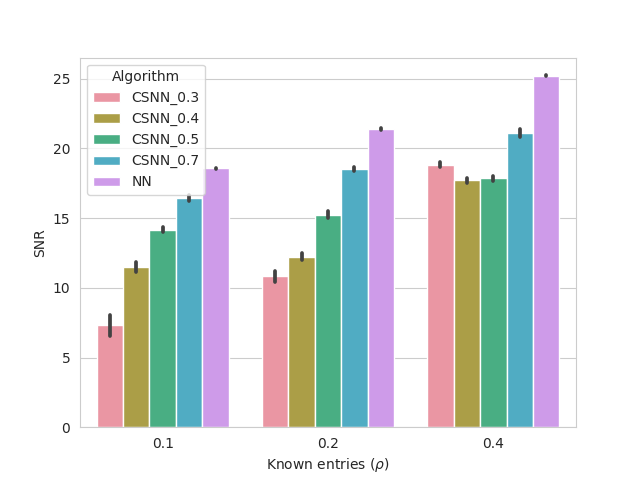} 
    \end{subfigure}
    \hfill
    \begin{subfigure}[b]{0.6\textwidth}
        \includegraphics[width=\textwidth]{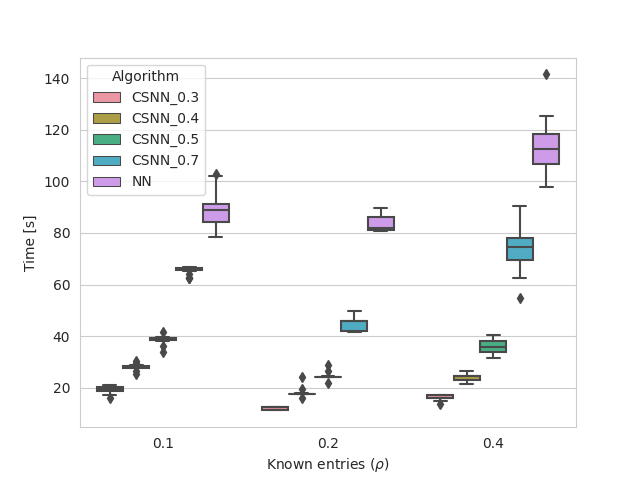} 
    \end{subfigure}

     \caption[Results P I]{Results for image inpainting; NN and CSNN-$\alpha$ algorithms for $\alpha~\in~\{ 0.3, 0.4, 0.5, 0.7 \}$.}
     \label{fig:pi_stats}
\end{figure*}

\begin{table*}[!htb]
\caption[Results P I]{Results for image inpainting; NN and CSNN-$\alpha$, $\alpha \in \{ 0.3, 0.4, 0.5, 0.7 \}$.}
\centering \begin{adjustbox}{width=\textwidth} 
\begin{tabular}{lrrrrrrrrr} \toprule {} & \multicolumn{3}{c}{SNR} & \multicolumn{3}{c}{Relative error ($\epsilon$)} & \multicolumn{3}{c}{Time [s]} \\ 
\cmidrule(lr){2-4} \cmidrule(lr){5-7} \cmidrule(lr){8-10} 
Known entries ratio $\rho$ & 0.4 & 0.2 & 0.1 & 0.4 & 0.2 & 0.1 & 0.4 & 0.2 & 0.1 \\  & & & & & & & & & \\ \midrule CSNN\_0.3 & 18.844 & 10.829 & 7.366 & 0.114 & 0.289 & 0.438 & 16.306 & 12.026 & 19.365 \\ CSNN\_0.4 & 17.708 & 12.245 & 11.491 & 0.130 & 0.245 & 0.268 & 24.151 & 17.997 & 28.123 \\ CSNN\_0.5 & 17.851 & 15.245 & 14.172 & 0.128 & 0.173 & 0.196 & 35.719 & 24.424 & 38.647 \\ CSNN\_0.7 & \textbf{21.134} & \textbf{18.526} & \textbf{16.458} & \textbf{0.088} & \textbf{0.119} & \textbf{0.151} & \textbf{73.670} & \textbf{44.948} & \textbf{65.639} \\ NN & \textbf{25.221} & \textbf{21.377} & \textbf{18.582} & \textbf{0.055} & \textbf{0.085} & \textbf{0.118} & \textbf{113.215} & \textbf{83.196} & \textbf{88.966} \\ \bottomrule \end{tabular} \end{adjustbox}  \label{tab:pi} \end{table*}

\newpage

\section{Proofs and supporting theorems}

We present a proof of Theorem \ref{thm:csmc_g}, accompanied by two additional theorems, \ref{thm:error_m0} and \ref{thm:strong_convex}, which are our original contributions. The proofs for these supplementary theorems are provided within this work. Additionally, the cited theorems from the literature have their proofs outlined therein.

We assume that $\mathbf{C} \in \mathbb{R}^{n_1 \times d}$ is a submatrix of $\mathbf{M} \in \mathbb{R}^{n_1 \times n_2}$.  $\mathbf{C}$ is obtained by sampling uniformly without replacement $d$ columns of $\mathbf{M}$. Let $r$ be the rank of $\mathbf{M}$ and let $\tilde{r}$ be the rank of $\mathbf{C}$. We  assume that the compact SVD of $\mathbf{C}$ is given by

\begin{align}\label{eq:svd_c}
\mathbf{C} = \mathbf{\tilde{U}}\mathbf{\tilde{\Sigma}}\mathbf{\tilde{V}^T},
\end{align}
\noindent
where $\mathbf{\tilde{U}} \in \mathbb{R}^{n_1 \times \tilde{r}}$, $\mathbf{\tilde{V}} \in \mathbb{R}^{n_2 \times \tilde{r}}$, and $\mathbf{\tilde{\Sigma}} \in \mathbb{R}^{\tilde{r} \times \tilde{r}}$. We  denote as $\mu_0(\mathbf{M})$ and  $\mu_0(\mathbf{C})$ the incoherence parameters of $\mathbf{M}$ and $\mathbf{C}$. We denote as $\tilde{U}$ the subspace spanned by the first $\tilde{r}$ left singular vectors of 
$\mathbf{C}$. The orthogonal projection onto $\tilde{U}$ is given by

\begin{align}
P_{\tilde{U}} = \mathbf{\tilde{U}}\mathbf{\tilde{U}}^T.
\end{align}

The following lemma can be found in \cite{Xu15} and shows that under the incoherence assumption, uniform sampling outputs a high-quality  solution to the CSS problem, with high quality.

\begin{theorem}[Theorem 9 in Xu et al. \cite{Xu15}]\label{thm:csincoherence}
    Let $\tilde{U}_{\psi}$ denote subspace spanned by the first $\psi$ left singular vectors of $\mathbf{C}$ and let $P_{\tilde{U}_{\psi}}$ denote orthogonal projection onto  $\tilde{U}_{\psi}$. Then for parameter $\gamma > 0$, with probability $1-2e^{-\gamma}$ we have
    \begin{align}
    \| \mathbf{M} - P_{\tilde{U}_{\psi}} \mathbf{M}\|_2^2 \leq \sigma_{\psi+1}^2 \left(1  + \frac{2n_2}{d}\right),
    \end{align}
\noindent
    if $d \geq  7 \mu(U_{\psi}) \psi (\gamma +\ln{\psi})$, where  $d$ denotes number of sampled columns in $\mathbf{C}$, $\sigma_{\psi+1}$ is the $(\psi+1)$-th largest singular value of $\mathbf{M}$ and $\mu(U_{\psi})$ is the incoherence parameter of the subspace spanned by $\psi$ first left singular vectors of $\mathbf{M}$,
\end{theorem}
\noindent
The following remark is the immediate consequence of the Theorem \ref{thm:csincoherence}.

\begin{remark}\label{rem:csincoherence}
Let  $\psi \geq r$, where $r$ is the rank of $\mathbf{M}$. Then $\sigma_{\psi+1}=\sigma_{r+1}=0$ and with probability $1-2e^{-\gamma}$,

\begin{align}
    \| \mathbf{M} - P_{\tilde{U}} \mathbf{M}\|^2_2 = 0,
\end{align}
provided that $d \geq  7 \mu_0(\mathbf{M})r(\gamma +\ln{r})$.
\end{remark}

 We now bound the distance measured in the spectral norm between $\mathbf{M}$ and $\mathbf{\hat{M}} = \mathbf{C}\mathbf{\hat{Z}}$. To do that, we assume that the objective function $g: \mathbb{R}^{\tilde{r} \times n_2} \rightarrow \mathbb{R}$,

\begin{align}\label{eq:g_def}
g(\mathbf{Y}) = \frac{1}{2} \| \mathcal{R}_{\Omega}(\mathbf{M}) - \mathcal{R}_{\Omega}(\mathbf{\tilde{U}}\mathbf{Y})\|_F^2
\end{align}
\noindent
is strongly convex \cite{boyd2004convex},

\begin{definition}[Boyd and Vandenberghe \cite{boyd2004convex}]
    A function $g:\mathbb{R}^{\tilde{r} \times n_2} \rightarrow \mathbb{R}$ is strongly convex with parameter $\beta > 0$ if $h(\mathbf{X})=g(\mathbf{X})-\frac{\beta}{2} \| \mathbf{X} \|_F^2$ is convex.
\end{definition}

The following remarks provide a helpful characterization of strongly convex functions.

\begin{remark}[Boyd and Vandenberghe \cite{boyd2004convex}]\label{rem:strong_cx_ineq}
A function $g: \mathbb{R}^{\tilde{r} \times n_2} \rightarrow \mathbb{R}$ is strongly convex with parameter $\beta > 0$ if $g$ is everywhere differentiable and
\begin{align}
g(\mathbf{Y}) \geq g(\mathbf{X}) + \langle \nabla g(\mathbf{X}), \mathbf{Y}-\mathbf{X} \rangle + \frac{\beta}{2} \| \mathbf{Y} - \mathbf{X}\|_F^2,
\end{align}
for any $\mathbf{X}, \mathbf{Y} \in \mathbb{R}^{\tilde{r} \times n_2}$, where inner product is defined as

\begin{align}
\langle \mathbf{X}, \mathbf{Y} \rangle = \text{tr}(\mathbf{X}^T\mathbf{Y}).
\end{align}

\end{remark}

\begin{remark}[Boyd and Vandenberghe \cite{boyd2004convex}]\label{rem:strong_hess}
If $g$ is twice differentiable, then $g(\mathbf{X})$ is strongly convex with parameter  $\beta > 0$ if  $\nabla^2g(\mathbf{X}) \succeq \beta \mathbf{I}$ for any $\mathbf{X} \in \mathbb{R}^{\tilde{r} \times n_2}$.
\end{remark}

We now provide the formula for the first-order and we partial derivatives of $g$. Let  $y_{sw}$ denote the $(s,w)$ entry of the matrix $\mathbf{Y}$, and $m_{lw}$  denote $(l,w)$ entry of $\mathbf{M}$ then the partial derivative of $g$ with respect to  $y_{sw}$ is given by

\begin{align}
\frac{\partial g}{\partial y_{sw}} &= -\sum \limits_{l : (l, w) \in \Omega} (m_{lw}-\sum_{i=1}^{\tilde{r}}{\tilde{u}_{li}} y_{iw}) \tilde{u}_{ls} \nonumber \\
&= \sum \limits_{l : (l, w) \in \Omega} (\sum_{i=1}^{\tilde{r}}{\tilde{u}_{li}} y_{iw} - m_{lw}) \tilde{u}_{ls}.
\end{align}

\noindent The second-order partial derivative is equal to

\begin{align}\label{eq:second_partial}
\frac{\partial^2 g}{\partial y_{sw}\partial y_{pq}} = \begin{cases}\sum \limits_{l : (l, w) \in \Omega} \tilde{u}_{ls}\tilde{u}_{lp} & w = q , \\
0 & w \neq  q .
\end{cases}
\end{align}

The following theorem shows that under certain assumptions, $\mathbf{\hat{M}} = \mathbf{C}\mathbf{\hat{Z}}$  is close to the matrix $\mathbf{M}$ in the spectral norm. Those assumptions include that i) $P_{\tilde{U}}\mathbf{M}$ is close to $\mathbf{M}$, ii) $g$ is strongly convex with parameter $\beta$. The following proofs take inspiration from the work of Xu et al. \cite{Xu15}.  

\begin{theorem}\label{thm:error_m0}
    Suppose that $\| \mathbf{M}- P_{\tilde{U}} \mathbf{M}\|_2^2 \leq \Delta$ for the parameter $\Delta > 0$, and let $\tilde{U}$ be the subspace of $\mathbb{R}^{n_1}$ spanned by the first $\tilde{r}$ left  singular vectors of $\mathbf{C}$, $P_{\tilde{U}} = \mathbf{\tilde{U}} \mathbf{\tilde{U}^T}$. Assume that $g$ is strongly convex with a parameter $\beta$,  then 

\begin{align}
\| \mathbf{M} -\mathbf{\hat{M}}\|_2^2 \leq 2\Delta + \frac{4 n_1 \Delta}{\beta}.
\end{align}
\end{theorem}

\begin{proof}
Set $\mathbf{Y}_1 = \mathbf{\mathbf{\tilde{U}}^T} \mathbf{M}$ , then from the assumptions

\begin{align}
    \|\mathbf{M} - \mathbf{\tilde{U}}\mathbf{Y}_1\|_2^2 \leq \Delta.
\end{align}

\noindent Thus

\begin{align}
   \frac{1}{r'}\|\mathbf{M} - \mathbf{\tilde{U}}\mathbf{Y}_1\|_F^2 \leq \|\mathbf{M} - \mathbf{\tilde{U}}\mathbf{Y}_1\|_2^2 \leq 
    \Delta,
\end{align}

\noindent where $r'$ denotes the rank of the matrix $(\mathbf{M} - \mathbf{\tilde{U}}\mathbf{Y}_1) \in \mathbb{R}^{n_1 \times n_2}$ and

\begin{align}
 \|\mathcal{R}_{\Omega}(\mathbf{M}) - \mathcal{R}_{\Omega}(\mathbf{\tilde{U}}\mathbf{Y}_1)\|_F^2 \leq \|\mathbf{M} -\mathbf{\tilde{U}}\mathbf{Y}_1\|_2^2 \leq r' \Delta.
\end{align}

\noindent Since any of the matrix dimensions bounds rank, $r' \leq n_1$ (for thick matrices $n_1 < n_2$)

\begin{align}\label{eq:help}
 \|\mathcal{R}_{\Omega}(\mathbf{M}) - \mathcal{R}_{\Omega}(\mathbf{\tilde{U}}\mathbf{Y}_1)\|_F^2 \leq  n_1 \Delta.
\end{align}

\noindent Let $\mathbf{\hat{Z}}$  be a solution to the problem \ref{eq:cs_incomplete}. Then 

\begin{align}
f(\mathbf{\hat{Z}}) &=  \|\mathcal{R}_{\Omega}(\mathbf{M}) - \mathcal{R}_{\Omega}(\mathbf{C}\mathbf{\hat{Z}})\|_F^2 
= \|\mathcal{R}_{\Omega}(\mathbf{M}) - \mathcal{R}_{\Omega}(\mathbf{\tilde{U}}\mathbf{\tilde{\Sigma}}\mathbf{\tilde{V}}^T\mathbf{\hat{Z}})\|_F^2,
\end{align}

\noindent and $\mathbf{\hat{Y}} := \mathbf{\tilde{\Sigma}}\mathbf{\tilde{V}}^T\mathbf{\hat{Z}}$ must be a minimum of $g$. 
Indeed, assume that $g(\mathbf{Y}_2) < g(\mathbf{\hat{Y}})$ for some $\mathbf{Y}_2 \neq \mathbf{\hat{Y}}$, i.e.

\begin{equation}
\begin{aligned}
 \|\mathcal{R}_{\Omega}(\mathbf{M}) - \mathcal{R}_{\Omega}(\mathbf{C}\mathbf{\hat{Z}})\|_F^2 &= \|\mathcal{R}_{\Omega}(\mathbf{M}) - \mathcal{R}_{\Omega}(\mathbf{\tilde{U}}\mathbf{\hat{Y}})\|_F^2 \\
 & > \|\mathcal{R}_{\Omega}(\mathbf{M}) - \mathcal{R}_{\Omega}(\mathbf{\tilde{U}}\mathbf{Y_2})\|_F^2 \\
 & = \|\mathcal{R}_{\Omega}(\mathbf{M}) - \mathcal{R}_{\Omega}(\mathbf{C}\underbrace{\mathbf{\tilde{V}}\mathbf{\tilde{\Sigma}^{-1}}\mathbf{Y_2}}_{\mathbf{X}_2})\|_F^2,
\end{aligned}
\end{equation}

\noindent and $f(\mathbf{X}_2) < f(\mathbf{\hat{Z}})$ where $f$ is the objective function in problem \ref{eq:cs_incomplete}. 

We bound distance between $\mathbf{Y}_1$ and $\mathbf{\hat{Y}}$ using strong convexity of $g$. Since $\mathbf{\hat{Y}}$ is minimum of $g(\mathbf{Y})$ (\ref{eq:g_def}), then $\nabla g (\mathbf{\hat{Y}})=\mathbf{0}$ and

\begin{align}
\langle \nabla g(\mathbf{\hat{Y}}), \mathbf{Y}_1 - \mathbf{\hat{Y}} \rangle = 0.
\end{align}

\noindent Thus, by the Remark \ref{rem:strong_cx_ineq}, we have

\begin{align}
g(\mathbf{Y}_1) \geq g(\mathbf{\hat{Y}}) + \frac{\beta}{2} \| \mathbf{Y}_1 - \mathbf{\hat{Y}}\|_F^2.
\end{align}

\noindent Using the definition of $g$ given in (\ref{eq:g_def}),

\begin{equation}\label{eq:sce}
\begin{aligned}
\frac{\beta}{2} \| \mathbf{Y}_1 - \mathbf{\hat{Y}} \|_F^2 
& \leq \|\mathcal{R}_{\Omega}(\mathbf{M}) - \mathcal{R}_{\Omega}(\mathbf{\tilde{U}}\mathbf{Y}_1)\|_F^2. \\
\end{aligned}
\end{equation}

\noindent Combining (\ref{eq:sce}) and  (\ref{eq:help}), we obtain

\begin{align}\label{eq:bound_yy0}
 \| \mathbf{Y}_1 - \mathbf{\hat{Y}}\|_F^2 \leq \frac{2}{\beta} \|\mathcal{R}_{\Omega}(\mathbf{M}) - \mathcal{R}_{\Omega}(\mathbf{\tilde{U}}\mathbf{Y}_1)\|_F^2 \leq \frac{2n_1}{\beta}\Delta.
\end{align}

\noindent By applying the triangle inequality, Cauchy-Schwarz inequality, and leveraging the fact that the Frobenius norm is an upper bound for the spectral norm, 

\begin{equation}
\begin{aligned}
\| \mathbf{M} -\mathbf{\hat{M}}\|_2^2   & = \| \mathbf{M}  -\mathbf{C}\mathbf{\hat{Z}} \|_2^2 \\
& \leq 2\| \mathbf{M} -\mathbf{\tilde{U}}\mathbf{\tilde{U}^T} \mathbf{M} \|_2^2 + 2\| \mathbf{\tilde{U}}\mathbf{\tilde{U}^T} \mathbf{M}  -\mathbf{C}\mathbf{\hat{Z}} \|_2^2 \\
& \leq 2\| \mathbf{M} - P_{\tilde{U}}\mathbf{M} \|_2^2 + 2\| \mathbf{\tilde{U}}\mathbf{\tilde{U}^T} \mathbf{M}  -\mathbf{C}\mathbf{\hat{Z}}  \|_F^2.
\end{aligned}
\end{equation}

\noindent The first component on the right side is bounded by the assumptions,

\begin{align}
 \| \mathbf{M} - P_{\tilde{U}}\mathbf{M} \|_2^2 \leq \Delta.
\end{align}

\noindent To bound the second component, we use the definition of the $\mathbf{Y}_1$ and $\mathbf{\hat{Y}}$,

\begin{equation}
\begin{aligned}
 \| \mathbf{\tilde{U}}\mathbf{\tilde{U}}^{T} \mathbf{M}  -\mathbf{\tilde{U}}\mathbf{\tilde{\Sigma}}\mathbf{\tilde{V}}^T\mathbf{\hat{Z}} \|_F^2  
=  \| \mathbf{\tilde{U}} (\mathbf{Y}_1 - \mathbf{\hat{Y}}) \|_F^2    \leq \| \mathbf{\tilde{U}}\|_2^2 \| \mathbf{Y}_1 - \mathbf{\hat{Y}} \|_F^2 = \| \mathbf{Y}_1 - \mathbf{\hat{Y}} \|_F^2.
\end{aligned}
\end{equation}

\noindent Using (\ref{eq:bound_yy0}), we obtain the following bound: 

\begin{align}
\| \mathbf{M} -\mathbf{\hat{M}}\|_2^2 \leq 2\Delta + \frac{4 n_1\Delta}{\beta}.
\end{align}

\end{proof}

To bound the parameter $\beta$ of the strong convexity, we use Remark \ref{rem:strong_hess} and bound the smallest eigenvalue of the Hessian of $g$. To do that, we use the following result of Tropp \cite{tropp10}.

\begin{theorem}[Theorem 5 in Xu et al. \cite{Xu15} derived from Theorem 2.2 in  Tropp \cite{tropp10} ]\label{thrm:tropp}
Let $\mathcal{X}$ be a finite set of the positive-semidefinite (PSD) matrices with dimension $k \times k$, and suppose that

\begin{align}
\max_{\mathbf{X} \in \mathcal{X}} \lambda_{\max}(\mathbf{X}) \leq B
\end{align}
 for some parameter $B>0$, where $\lambda_{\max}(\mathbf{X})$ is the maximum eigenvalue of $\mathbf{X}$.
Sample $\{ \mathbf{X}_1, \ldots, \mathbf{X}_{\Psi}\}$ uniformly at random from $\mathcal{X}$ without replacement. Compute:

\begin{align}
\mu_{\min} := \Psi \lambda_{\min}(\mathbb{E}\mathbf{X}_1),
\end{align}

\noindent and

\begin{align}
\mu_{\max} := \Psi \lambda_{\max}(\mathbb{E}\mathbf{X}_1),
\end{align}

\noindent where $\mathbb{E}\mathbf{X}_1$ is the expected value of a random variable $\mathbf{X}_1$, $\lambda_{\max}(\mathbb{E}\mathbf{X}_1)$ and $\lambda_{\min}(\mathbb{E}\mathbf{X}_1)$ denote its maximum and minimum eigenvalue.

\begin{align}
P\big(\lambda_{\max}(\sum_{j=1}^{\Psi} \mathbf{X}_j) \! \geq \! (1 \! + \! \rho)\mu_{\max}\big) \! \nonumber \\ \leq \! k \exp{\frac{-\mu_{\max}}{B}} [(1\!+\!\rho)\ln{\!(1\!+\!\rho)} \!-  \!\rho],
\end{align}

\noindent for parameter $\rho \in [0, 1)$.

\begin{flalign}
 P\big(\lambda_{\min}(\sum_{j=1}^{\Psi} \mathbf{X}_j) \!  \leq  \! (1 \! - \! \rho)\mu_{\min}\big) \!\nonumber \\
\leq   k \exp{\frac{-\mu_{\min}}{B}} [(1\! -\! \rho)\ln{(1\!-\!\rho)} \! + \! \rho],
\end{flalign}

\noindent for parameter $\rho \geq 0$.
\end{theorem}

\begin{theorem}\label{thm:strong_convex}
Let $\gamma > 0$ be a parameter. With probability $1-e^{-\gamma}$ we have that the objective function $g : \mathbb{R}^{\tilde{r} \times n_2} \rightarrow \mathbb{R}$ defined in  (\ref{eq:g_def})  is strongly convex with parameter $\beta>0$, provided that
\begin{align}
| \Omega | \geq \tilde{r} n_2 \mu(\tilde{U})\left(\gamma+\ln{(\frac{n_2\tilde{r}}{2})}\right).
\end{align}
\end{theorem}
\begin{proof}
By remark \ref{rem:strong_hess} to bound strong convexity, we can bound the smallest eigenvalue of the Hessian matrix $\mathbf{H}=\nabla^2g$.

The Hessian matrix of a function $g$ is a $\tilde{r}n_2 \times \tilde{r}n_2$ matrix. Let us assume that second-order derivative with respect to the $y_{sw}$ and $y_{pq}$ entries of matrix $\mathbf{Y}$ is the $(\tilde{r}(s-1)+w$, $\tilde{r}(p-1)+q)$ entry of the Hessian matrix. Then using (\ref{eq:second_partial}) the Hessian matrix $\mathbf{H}$ can be written as

\begin{align}
  \begin{pmatrix}       
    \mathbf{H}^{1,1} & \mathbf{H}^{1,2} & \ldots & \mathbf{H}^{1,{\tilde{r}}}\\
     \mathbf{H}^{2,1} & \ddots  & \ldots &  \mathbf{H}^{2,{\tilde{r}}}\\
     \vdots & \vdots & \ddots  & \vdots   \\
     \mathbf{H}^{{\tilde{r}}, 1} & \mathbf{H}^{{\tilde{r}}, 2} & \ldots & \mathbf{H}^{{\tilde{r}}, \tilde{r}}  \\
   \end{pmatrix},
\end{align}

\noindent where $\mathbf{H}^{s, q}$ is a diagonal $n_2\times n_2$ matrix containing partial derivatives  $\frac{\partial^2 g}{\partial y_{sw}\partial y_{pq}}$ for $w, q \in \{1, \ldots, n_2\}$.

\begin{align}
\mathbf{H}^{s, q} =  \begin{pmatrix}       
    \sum \limits_{l: (l,1) \in \Omega} \tilde{u}_{ls}\tilde{u}_{lp} & 0 & \ldots & 0 \\
     0& \ddots  & \ldots &  0\\
     \vdots & \vdots & \ddots  & \vdots   \\
     0 &0 & \ldots & \sum \limits_{l: (l,n_2)\in \Omega}\tilde{u}_{ls}\tilde{u}_{lp}  \\
   \end{pmatrix}
\end{align}

\noindent Then the Hessian of $g$ is a sum of the random matrices,

\begin{align}
\mathbf{H}  = \sum\limits_{(i,j) \in \Omega} \mathbf{\tilde{u}}_{i:} \mathbf{\tilde{u}}_{i:}^T  \otimes   \mathbf{e}_j\mathbf{e}_j^{T},
\end{align}
where $\mathbf{e}_j \in \mathbb{R}^{n_2}$ is a standard basis vector and $\mathbf{\tilde{u}}_{i:} \in \mathbb{R}^{r}$ is a vector defined by the $i$-th row of matrix $\mathbf{\tilde{U}}$, and $\otimes$ denote the Kronecker product. Thus the Hessian of $g$ is a sum of the $|\Omega |$ random matrices of the form
\begin{align}
\mathbf{H}^{i,j} :=\mathbf{\tilde{u}}_{i:} \mathbf{\tilde{u}}_{i:}^T  \otimes   \mathbf{e}_j\mathbf{e}_j^T,
\end{align}
\noindent
where $\mathbf{\tilde{u}}_{i:} \mathbf{\tilde{u}}_{i:}^T \in \mathbb{R}^{\tilde{r} \times \tilde{r}}$ and $ \mathbf{e}_j\mathbf{e}_j^T \in \mathbb{R}^{n_2 \times n_2}$ are PSD. Thus,

\begin{align}\label{eq:hsum}
\mathbf{H}  = \sum\limits_{(i,j) \in \Omega} \mathbf{H}^{i,j}.
\end{align}

\noindent Each $ \mathbf{H}^{i,j}$ is PSD as the Kronecker product of the two PSD matrices. Moreover,


\begin{equation}
\begin{aligned}
\lambda_{\max}(\mathbf{H}^{i,j}) &=  \lambda_{\max}(\mathbf{\tilde{u}}_{i:}\mathbf{\tilde{u}}_{i:} ^T )  = \sum_{j=1}^{\tilde{r}} \tilde{u}_{ij}^2 = \| \mathbf{\tilde{u}}_{i:}\|_2^2  \leq \frac{\tilde{r}\mu(\tilde{U})}{n_1}
\end{aligned}
\end{equation}

\noindent for each $i = 1, \cdots n_1$, $j=1, \cdots n_2$. Let $(i_1, j_1)$ be the first pair of indices in $\Omega$, i.e., the indices of the first observed entry in $\mathbf{M}$.  The expected value of the random matrix $\mathbf{H}^{i_1,j_1}$ is given by

\begin{equation}
\begin{aligned}
\mathbb{E}(\mathbf{H}^{i_1,j_1}) &= \frac{1}{n_1 n_2} \sum \limits_{l=1}^{n_1}  \sum \limits_{q=1}^{n_2} \mathbf{H}^{l,q} \\
& = \frac{1}{n_1 n_2} \mathbf{\tilde{U}}^T\mathbf{\tilde{U}} \otimes  \mathbf{I}_{n_1 \times n_1} \\ & = \frac{1}{n_1 n_2}  \mathbf{I}_{\tilde{r} \times \tilde{r}} \otimes \mathbf{I}_{n_1 \times n_1} \\ & =  \frac{1}{n_1 n_2} \mathbf{I}_{\tilde{r}n_1 \times \tilde{r}n_1} .
\end{aligned}
\end{equation}

\noindent Following the notation of Theorem \ref{thrm:tropp},

\begin{align}
\mu_{\min} = | \Omega | \lambda_{\min} (\mathbb{E}(\mathbf{H}^{i_1,j_1})) = \frac{| \Omega |}{n_1 n_2},
\end{align}
\noindent and 

\begin{align}
 \max_{ij} \lambda_{\max}(\mathbf{H}^{i,j}) \leq  \frac{\tilde{r}\mu(\tilde{U})}{n_1}= B. 
\end{align}

\noindent Combining Theorem \ref{thrm:tropp} with $\rho=\frac{1}{2}$ and (\ref{eq:hsum})

\begin{align}
P(\lambda_{\min}(\mathbf{H})  \leq  \frac{1}{2}\mu_{\min})  & \leq  \tilde{r}n_2 \exp{\left(\frac{-\mu_{\min}}{B}\right)} \left(\frac{1-\ln{2}}{2}\right) \nonumber \\ & \leq  \frac{\tilde{r}n_2}{2} \exp{\left(\frac{-\mu_{\min}}{B}\right)},
\end{align}

\noindent implying

\begin{align}
P\left (\lambda_{\min}(\mathbf{H})  \leq   \frac{| \Omega |}{2n_1 n_2}\right )  \leq \frac{\tilde{r}n_2}{2} \exp{\left (- \frac{| \Omega |}{ \tilde{r}n_2\mu(\tilde{U})} \right )},
\end{align}

\noindent Hence, with probability at least $1-e^{-\gamma}$, 

\begin{align}
   \lambda_{\min}(\mathbf{H}) \geq \frac{|\Omega|}{2n_1n_2},
\end{align}

\noindent provided that

\begin{align}
| \Omega | \geq \tilde{r} n_2 \mu(\tilde{U})\left(\gamma+\ln{(\frac{n_2\tilde{r}}{2})}\right).
\end{align}
\end{proof}
\noindent
Theorem \ref{thm:csmc_g} can be proved by combining the results of Theorems \ref{thm:error_m0} and \ref{thm:strong_convex}.

\begin{proof}[Proof of Theorem \ref{thm:csmc_g}]

Since $d \geq 7 \mu_0(\mathbf{M})r (\gamma + \ln r)$, from the Remark \ref{rem:csincoherence}, 

\begin{align}\label{eq:ttt}
    \| \mathbf{M} - P_{\tilde{U}} \mathbf{M}\|_2^2 = 0,
\end{align}

\noindent with the probability at least $1-2e^{-\gamma}$. 
From Theorem \ref{thm:strong_convex}, the fact that $\mu_0(\mathbf{C})>\mu(\tilde{U})$ 
and the fact $| \Omega | \geq \tilde{r} n_2 \mu_0(\mathbf{C})\left(\gamma+\ln{(\frac{n_2\tilde{r}}{2})}\right)$, function $g$ is $\beta$-strongly convex with probability at least $1-e^{-\gamma}$. We can apply Theorem \ref{thm:error_m0}  with  $\Delta = 0$ and show that

\begin{align}
\| \mathbf{M} -\mathbf{\hat{M}}\|_2^2 \leq 0,
\end{align}

\noindent implying $\mathbf{M} = \mathbf{\hat{M}}$ with probability at least $1-3e^{-\gamma}$.

\end{proof}

\section{Conclusion}

This paper introduces the Columns Selected Matrix Completion (CSMC) method, which enhances the time efficiency of nuclear norm minimization algorithms for low-rank matrix completion. CSMC employs a two-stage approach: first, it applies a low-rank matrix completion algorithm to a reduced problem size, followed by solving a least squares problem to reconstruct the full matrix. Theoretical analysis provides provable guarantees for perfect recovery under standard assumptions.

Our primary focus is on improving the computational efficiency of convex matrix completion methods while preserving reconstruction quality. To this end, we assume that observed entries follow a uniform sampling distribution and provide in-depth theoretical analysis under this setting. Experimental results on both synthetic and real-world datasets confirm that CSMC achieves competitive accuracy with substantially reduced runtime. The method is especially well-suited for large, incomplete matrices arising in practical applications such as recommendation systems and image inpainting.

While this study focuses on uniform sampling and convex formulations, the CSMC framework is flexible and can be extended in several directions. Future research will explore alternative sampling strategies, such as adaptive column sampling in the first stage, to relax the incoherence assumption and improve performance on coherent matrices. Another promising direction is to study robustness, both to different types of noise and to error propagation between the two stages of CSMC. Finally, applying Column Subset Selection techniques to structured matrix completion will broaden the applicability of the CSMC framework \cite{usevich2025structured, Cai23}.
\\[1em]
\noindent \textbf{Acknowledgments:} The authors thank Dr. Aikaterini Aretaki for her critical review and comments on the manuscript. They also thank Professor Dariusz Uciński for his support, in particular with the implementation of CUR+Nuc.
\\[1em]

\bibliographystyle{unsrt}  
\bibliography{sn-article}  
\end{document}